\documentclass[accepted]{uai2026} % after acceptance, for a revised version; 
\usepackage[british]{babel}

\usepackage{natbib} % has a nice set of citation styles and commands
\usepackage{mathtools} % amsmath with fixes and additions
\usepackage{booktabs} % commands to create good-looking tables
\usepackage{tikz} % nice language for creating drawings and diagrams
\usepackage{microtype}
\usepackage{graphicx}
\graphicspath{{images/}} % images are in uai2026-template/images/
\usepackage{subcaption}
\usepackage{dsfont}
\usepackage{bbm}
\usepackage{hyperref}
\usepackage{fontawesome5}

\usepackage{amsmath}
\usepackage{amssymb}
\usepackage{mathtools}
\usepackage{amsthm}
\usepackage{wrapfig}

\usepackage[capitalize,noabbrev]{cleveref}
\usepackage[algo2e,linesnumbered,ruled,vlined]{algorithm2e}
\theoremstyle{plain}
\newtheorem{theorem}{Theorem}[section]
\newtheorem{proposition}[theorem]{Proposition}
\newtheorem{lemma}[theorem]{Lemma}
\newtheorem{corollary}[theorem]{Corollary}
\theoremstyle{definition}

\newtheorem{assumption}[theorem]{Assumption}
\theoremstyle{remark}
\newtheorem{remark}[theorem]{Remark}

\usepackage[textsize=tiny]{todonotes}

\title{Causal EpiNets: Precision-corrected Bounds on Individual Treatment Effects using Epistemic Neural Networks}

\author[1]{\href{mailto:<gandharv.patil@mail.mcgill.ca>?Subject=Your UAI 2026 paper}{Gandharv Patil}{}\thanks{Work done at RBC Borealis}}
\author[2]{Keyi Tang}
\author[2]{Raquel Aoki}
\author[2]{Leo Guelman}

\affil[1]{%
    McGill University/Mila
}
\affil[2]{%
    RBC Borealis
}

\begin{document}
\maketitle

\begin{abstract}
Individual treatment effects are not point-identified from data. 
The Probability of Necessity and Sufficiency (PNS) circumvents 
this limitation by characterizing individual-level causality through 
intersection bounds derived from combined experimental and observational 
data. In finite samples, however, standard plug-in estimators systematically 
fail: they violate structural probability constraints and suffer from extremum 
bias induced by max–min operators, yielding spuriously narrow intervals. 
We propose a neural framework for finite-sample PNS estimation that resolves 
both pathologies. We introduce an anchored neural architecture that guarantees 
structural constraint satisfaction by construction. To correct extremum bias, 
we employ precision-corrected intersection-bound inference, leveraging Epistemic 
Neural Networks for scalable, high-dimensional uncertainty quantification. 
Empirical evaluations confirm that this approach maintains nominal coverage 
and exact constraint validity in high-dimensional regimes where standard 
estimators systematically undercover.
\end{abstract}

\section{Introduction}
Estimating causal effects at the level of individuals is central in domains
such as healthcare, policy design, and legal decision-making.
While modern machine learning methods have made substantial progress in
estimating conditional average treatment effects (CATEs),
average effects do not directly answer the question that often motivates
individual-level decisions:
\emph{what is the probability that treatment would help this specific
individual?}

Average effects mask heterogeneity in individual responses.
An average effect of zero, for instance, is consistent with both universal
null effects and a mixture of strong positive and negative responders.
For many decisions, the relevant quantity is not an average over similar
individuals, but the probability that treatment would change the outcome
for this individual.

The \emph{Probability of Necessity and Sufficiency} (PNS),
introduced by \citet{tian2000probabilities},
provides a principled characterization of this notion.
In the binary treatment and outcome setting, PNS is the probability that an
individual would experience the outcome under treatment but not under
control.
A key result of \citet{tian2000probabilities} shows that PNS is not point 
identified, but can be sharply bounded by combining experimental and observational data. 
Experimental data identify interventional response probabilities, 
while observational data constrain the joint distribution of treatment and outcome. 
Together, these sources restrict the set of admissible structural causal models. 
A profound consequence of this mathematical structure is that it allows for the 
computation of bounds on the PNS without requiring strong, untestable assumptions 
such as unconfoundedness or positivity—assumptions that otherwise dominate 
the causal inference literature.
These identification results, however, assume exact knowledge of the
relevant probabilities.
In practice, the bounds must be estimated from finite samples, often with
high-dimensional covariates.
A natural approach is to estimate each probability from data and substitute 
these estimates into the bound expressions. Such naive plug-in estimators, however, 
suffer from two fundamental pathologies.
First, they yield probability estimates that are incompatible with any
underlying structural causal model.
Second, PNS bounds are defined through maxima and minima over multiple
estimated terms.
This intersection-bound structure induces systematic extremum bias, causing
lower bounds to be overestimated and upper bounds to be underestimated
\citep{chernozhukov2013intersection}.
Valid inference therefore requires procedures that explicitly account for
the envelope structure of the bounds.

We present the first complete framework for finite-sample estimation and inference of Tian–Pearl bounds on the Probability of Necessity and Sufficiency (PNS) with high-dimensional covariates. Our approach integrates three key components. First, we introduce an anchored, constrained multi-head neural architecture that jointly models observational and interventional quantities while guaranteeing compatibility with probability axioms and a shared structural causal model by construction. Second, we show how PNS bounds fit naturally into the intersection-bounds framework and adapt precision-corrected inference to this setting, yielding confidence intervals that explicitly account for bias induced by max–min operators. Third, we develop a scalable uncertainty quantification procedure based on Epistemic Neural Networks (ENN) \citep{osband2023epistemic}, which provides a practical surrogate for resampling-based inference and enables precision correction in modern over-parameterised models. Together, these components provide the first end-to-end approach to finite-sample PNS-bound estimation with modern machine learning models, and establish a principled baseline for individual-level causal inference beyond average effects.

The remainder of the paper is organized as follows. Section 2 introduces the problem setup and reviews population-level PNS bounds. Section 3 presents the proposed anchored architecture, uncertainty modelling, and precision-corrected inference procedure. Section 4 reports empirical results. Section 5 discusses related work, and Section 6 concludes the paper.

\section{Problem Setup}
\label{sec:pns-setup}

We consider a binary treatment $X \in \{0,1\}$, a binary outcome
$Y \in \{0,1\}$, and pre-treatment covariates $Z \in \mathcal{Z}$.
We work within a structural causal model in which the potential
outcomes $Y_x$ and $Y_{x'}$ are well defined, corresponding to the
outcomes that would be observed under the interventions
$\mathrm{do}(X=1)$ and $\mathrm{do}(X=0)$, respectively.

For each covariate value $z$, we summarize the relevant conditional
quantities as follows. The \emph{interventional success probabilities}
are
\begin{equation}
\begin{aligned}
\mu_1(z) &:= P\!\left(Y = 1 \mid \mathrm{do}(X = 1), Z = z\right), \\
\mu_0(z) &:= P\!\left(Y = 1 \mid \mathrm{do}(X = 0), Z = z\right),
\end{aligned}
\end{equation}
and the \emph{observational joint probabilities} are
\begin{equation}
p_{xy}(z) := P(X = x, Y = y \mid Z = z), \ \  x,y \in \{0,1\},
\end{equation}
so that $p_{10}(z) + p_{11}(z) = P(X = 1 \mid Z = z)$ and
$p_{01}(z) + p_{00}(z) = P(X = 0 \mid Z = z)$.

\subsection{Probability of necessity and sufficiency}

Following Tian and Pearl~\cite{tian2000probabilities}, the
\emph{probability of necessity and sufficiency} (PNS) of $X$ for $Y$
is defined as
\begin{equation}
\mathrm{PNS}
:= P\!\left(Y_x = 1,\; Y_{x'} = 0\right).
\label{eq:pns-def}
\end{equation}
This quantity represents the fraction of units for which the treatment
is both necessary and sufficient for the outcome.

We define the covariate-conditional PNS as
\begin{equation}
\mathrm{PNS}(z)
:= P\!\left(Y_x = 1,\; Y_{x'} = 0 \mid Z = z\right),
\label{eq:pns-z-def}
\end{equation}
and view the marginal PNS as 
$\mathrm{PNS} = \mathbb{E}_Z\!\left[\mathrm{PNS}(Z)\right]$.

\subsubsection{Structural assumptions and constraint system}
\label{subsec:pns-assumptions}

We adopt the structural assumptions used in
\cite{tian2000probabilities}, formulated conditionally on $Z = z$.

\begin{assumption}[{\bf Experimental identification}]
In the experimental regime, treatment assignment is randomized, possibly
conditional on $Z$, so that $X$ is independent of the potential outcomes
given $Z$.
Under this protocol, the interventional probabilities are identified by
\begin{align}
\nonumber \mu_x(z)
&:= P(Y = 1 \mid \mathrm{do}(X = x), Z = z)\\
&= P^{\mathrm{exp}}(Y = 1 \mid X = x, Z = z),
\qquad x \in \{0,1\}.
\end{align}
\end{assumption}

\begin{assumption}[{\bf Stability across regimes}]
Experimental and observational data share the same structural equation
for $Y$ and the same distribution of exogenous variables.
The treatment assignment mechanism for $X$ may differ across regimes:
in the experimental regime, $X$ is assigned at random, while in the
observational regime $X$ may depend on $Z$ and unobserved factors.
Under this shared structural causal model, the observational joint
distribution $(X,Y)\mid Z$ restricts the set of admissible interventional
probabilities $(\mu_0(z), \mu_1(z))$ through the compatibility constraints
\begin{equation}
\begin{aligned}
p_{11}(z) &\le \mu_1(z) \le 1 - p_{10}(z), \\
p_{01}(z) &\le \mu_0(z) \le 1 - p_{00}(z),
\end{aligned}
\label{eq:observational-constraints}
\end{equation}
which must hold for any values of $(\mu_0(z), \mu_1(z))$ consistent with
the observational distribution.
\end{assumption}

\subsubsection{Conditional and marginal PNS bounds}
\label{subsec:pns-bounds}

For a fixed $z$, the conditional PNS is a linear functional of the
latent probabilities. Tian and Pearl derive sharp bounds by solving the
corresponding linear programs.

In our notation, the bounds are
\begin{equation}
\begin{aligned}
\mathrm{PNS}_{\mathrm{L}}(z)
:= \max\Big\{&
0,\;
\mu_1(z) - \mu_0(z), \\
& p_{11}(z) + p_{01}(z) - \mu_0(z), \\
& \mu_1(z) - p_{11}(z) - p_{01}(z)
\Big\},
\end{aligned}
\label{eq:pns-lower-1}
\end{equation}

\begin{equation}
\begin{aligned}
\mathrm{PNS}_{\mathrm{U}}(z)
:= \min\Big\{&
\mu_1(z),\;
1 - \mu_0(z), \\
& p_{11}(z) + p_{00}(z), \\
& \mu_1(z) - \mu_0(z)
  + p_{10}(z) + p_{01}(z)
\Big\}.
\end{aligned}
\label{eq:pns-upper-1}
\end{equation}

These satisfy
\begin{equation}
\mathrm{PNS}_{\mathrm{L}}(z)
\le \mathrm{PNS}(z)
\le \mathrm{PNS}_{\mathrm{U}}(z),
\end{equation}
and the interval is sharp, and
the marginal PNS satisfies
$\mathrm{PNS}
\in \big[\beta_{\mathrm{L}}, 
\beta_{\mathrm{U}}\big]$,
where $\beta_{\mathrm{L}} = \mathbb{E}_{Z}[\mathrm{PNS}_{\mathrm{L}}(z)]$ and $\beta_{\mathrm{U}} = \mathbb{E}_{Z}[\mathrm{PNS}_{\mathrm{U}}(z)]$.

The same framework extends to bounds on the probability of harm
$P(Y_{x'} = 1, Y_x = 0 \mid Z = z)$.
In the binary setting, $\mathrm{PNS}(z)$ coincides with the probability
that the individual treatment effect is positive.
Since this quantity is not point-identified, the Tian--Pearl bounds
provide the tightest possible interval compatible with the observed
data.

\section{Proposed Approach/Contribution}

While exact in the population limit, Tian–Pearl bounds must in practice 
be estimated from finite samples. A naive plug-in approach—fitting separate 
models for experimental and observational distributions—suffers from two 
fundamental structural flaws. First, independently estimated models for $p_{xy}(z)$ 
and $\mu_x(z)$ routinely violate the cross-regime feasibility constraints 
(Eq.~\ref{eq:observational-constraints}). Such estimates fail the compatibility 
requirements of a shared structural causal model, yielding internally inconsistent bounds.

\begin{figure}[ht]
  \centering
  \includegraphics[width=0.8\columnwidth]{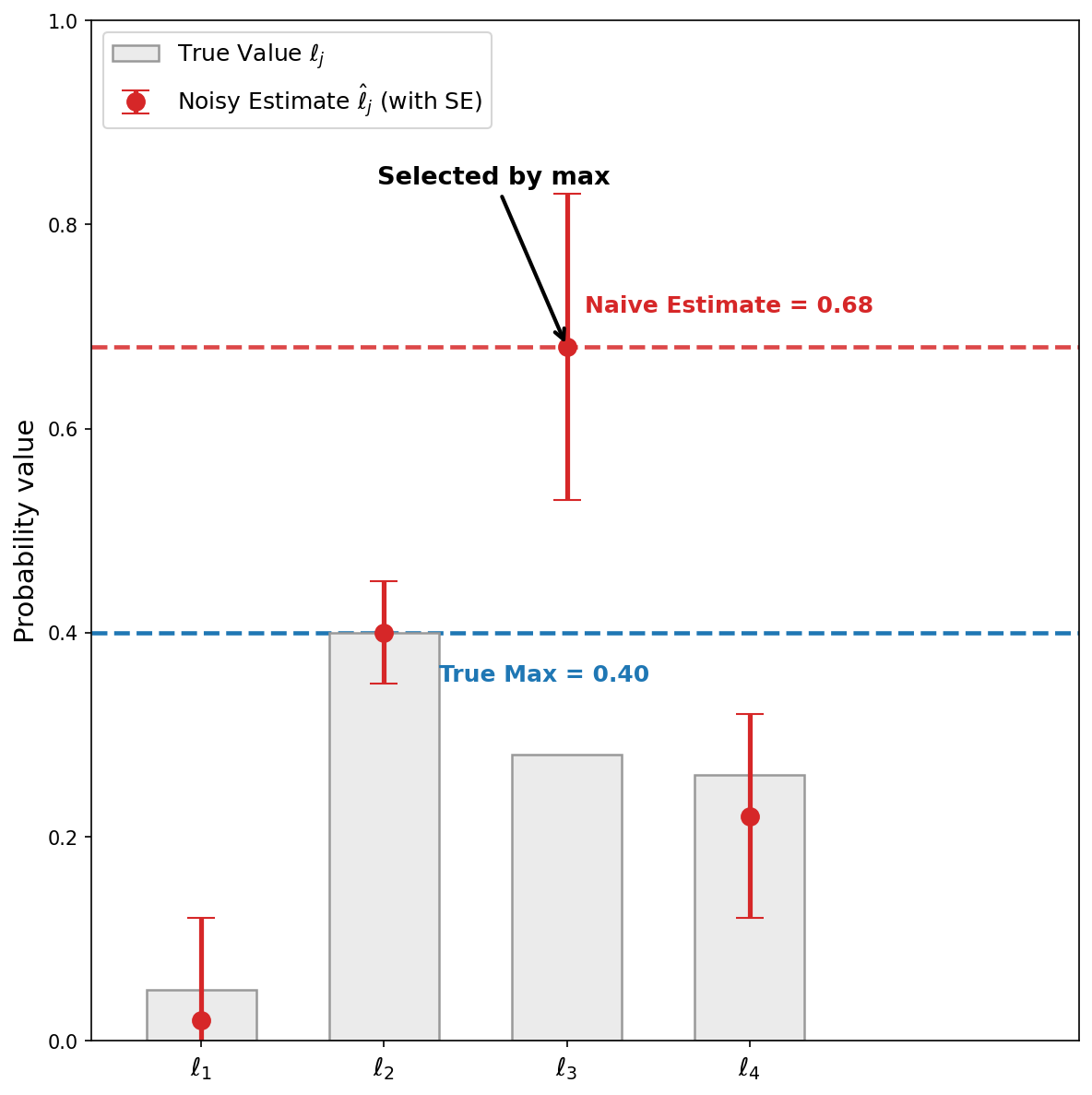}
  \caption{Overestimation bias in the lower bound.}
  \label{fig:bias_illustration}
  \vspace{-1em} % Adjust this value (e.g., -0.1in, -10pt) if you need it tighter
\end{figure}

Second, the max/min structure of PNS bounds induces systematic 
extremum bias even with consistent probability estimates. 
To illustrate (Fig.~\ref{fig:bias_illustration}), let the 
true lower-bound terms be $\ell \in \{0.05, 0.40, 0.28, 0.30\}$, 
yielding $\mathrm{PNS}_L(z) = 0.40$. A noisy estimator yielding 
$\hat{\ell} \in \{0.02, 0.40, 0.68, 0.22\}$ produces a plug-in 
bound of $\hat{\mathrm{PNS}}_L(z) = \max \hat{\ell} = 0.68 > 0.40$. 
Because extremum operators naturally select the most severe estimation 
errors, this "winner's curse" systematically overestimates lower bounds 
and underestimates upper bounds, yielding spuriously tight 
intervals. Valid inference therefore demands a framework that 
guarantees structural causal compatibility, quantifies finite-sample 
uncertainty, and rigorously corrects for extremum bias. 
We achieve this by unifying a constrained multi-head architecture 
with scalable neural bootstrap surrogates \citep{osband2023epistemic} 
and precision-corrected intersection bounds.

\subsection{Ensuring estimates satisfy probability axioms}

The structural constraints (Eq.~\ref{eq:observational-constraints})
impose algebraic relationships between the observational joint probabilities
$p_{xy}(z)$ and the interventional probabilities $\mu_x(z)$.
Rather than estimating these six quantities independently, we design an
architecture that enforces these relationships by construction.
The key idea is to treat the joint probabilities $p_{xy}(z)$ as \emph{anchors}:
primary outputs trained directly on observational data.
The interventional probabilities $\mu_x(z)$ are then parameterized as
constrained functions of these anchors.

Our architecture consists of a shared representation network
$\phi_\theta : \mathcal{Z} \to \mathbb{R}^d$ that maps covariates to a learned
embedding, followed by three output heads.
The first head outputs logits over the four joint outcomes
$(X,Y) \in \{0,1\}^2$, which are passed through a softmax to produce the joint
probabilities $p_{00}(z), p_{01}(z), p_{10}(z), p_{11}(z)$.
The remaining two heads output unconstrained scalar values
$\delta_0(z), \delta_1(z) \in \mathbb{R}$, which are transformed into the
interventional probabilities $\mu_0(z), \mu_1(z)$ via a
constraint-enforcing reparameterization layer.

We enforce the Tian--Pearl constraints by reparameterizing the interventional
probabilities as convex combinations of their feasible bounds:
\begin{align}
\mu_0(z)
  &= p_{01}(z)
     + \bigl(1 - p_{00}(z) - p_{01}(z)\bigr)\,\sigma\bigl(\delta_0(z)\bigr),
  \\
\mu_1(z)
  &= p_{11}(z)
     + \bigl(1 - p_{10}(z) - p_{11}(z)\bigr)\,\sigma\bigl(\delta_1(z)\bigr),
\end{align}
where $\sigma(\cdot)$ is the sigmoid function.
For each fixed $z$, the map $\delta_x(z) \mapsto \mu_x(z)$ is smooth and
strictly increasing, since $\mu_x(z)$ is an affine function of
$\sigma\bigl(\delta_x(z)\bigr)$ and $\sigma$ is smooth and strictly increasing
on $\mathbb{R}$.
Writing the lower and upper feasible bounds for $\mu_x(z)$ as
$\ell_x(z)$ and $u_x(z)$, the derivative
$\partial \mu_x(z) / \partial \delta_x(z)
  = \bigl(u_x(z) - \ell_x(z)\bigr)\,
    \sigma\bigl(\delta_x(z)\bigr)\bigl(1 - \sigma\bigl(\delta_x(z)\bigr)\bigr)$
is strictly positive whenever the interval $[\ell_x(z), u_x(z)]$ is
non-degenerate.
Thus the reparameterization is a smooth, order-preserving bijection from
$\mathbb{R}$ onto the open feasible interval for $\mu_x(z)$.
Gradient-based optimization over $\delta_x(z)$ therefore explores the entire
feasible region for $\mu_x(z)$ while preserving the constraints.
Since $\sigma(\delta) \in (0,1)$, each $\mu_x(z)$ lies within its feasible
interval for any network weights, and the Tian--Pearl consistency constraints
hold automatically.

The joint-probability head is trained on observational data using a
cross-entropy loss over the four-class outcome $(X,Y)$.
The auxiliary heads are trained on experimental data using binary
cross-entropy losses on the derived interventional probabilities
$\mu_0(z)$ and $\mu_1(z)$.
We keep the gradients from these losses separate: the experimental loss
updates only the auxiliary heads, treating the joint probabilities as fixed
when computing the feasible bounds.
This preserves a clean anchoring of $\mu_x(z)$ to $p_{xy}(z)$ while still
allowing the representation $\phi_\theta$ to be shared across tasks.

\subsection{Bias correction and uncertainty quantification}

Our inference procedure builds on the intersection bounds methodology of
\cite{chernozhukov2013intersection}.
Let $\psi^\star$ denote a scalar target parameter that is only partially
identified.
Suppose $\psi^\star$ is known to lie in the intersection
\begin{align}
\psi^\star
  \in \bigcap_{v \in \mathcal{V}}
     \bigl[\,\psi_\ell(v),\,\psi_u(v)\,\bigr],
\end{align}
where $\psi_\ell(v)$ and $\psi_u(v)$ are lower and upper bounding functions
indexed by $v \in \mathcal{V}$.
Equivalently, the identified set for $\psi^\star$ is the interval
$[\psi_\ell^\star, \psi_u^\star]$ with
\begin{align}
\psi_\ell^\star
  = \sup_{v \in \mathcal{V}} \psi_\ell(v),
\qquad
\psi_u^\star
  = \inf_{v \in \mathcal{V}} \psi_u(v).
\end{align}

Given estimators $\hat{\psi}_\ell(v)$ and $\hat{\psi}_u(v)$ with pointwise
standard errors $s_\ell(v)$ and $s_u(v)$, a naive plug-in approach forms
$\sup_{v} \hat{\psi}_\ell(v)$ and $\inf_{v} \hat{\psi}_u(v)$.
These plug-in estimators are systematically biased.
The supremum tends to pick out values of $v$ where estimation noise is positive
($\hat{\psi}_\ell(v) > \psi_\ell(v)$), producing an upward-biased estimate of the
lower bound.
Conversely, the infimum tends to pick out values where noise is negative,
producing a downward-biased estimate of the upper bound.
The resulting bounds are therefore spuriously tight relative to the true
identified set.

The precision-corrected estimators of \cite{chernozhukov2013intersection}
adjust for this \textbf{optimization bias} by penalizing regions of high uncertainty.
The uncertainty at each $v$ is summarized by $s_\ell(v)$ and $s_u(v)$, which
measure the sampling variability of $\hat{\psi}_\ell(v)$ and
$\hat{\psi}_u(v)$.
The precision-corrected lower and upper bound estimators are defined as

\begin{minipage}{0.48\linewidth}
\begin{equation}
\hat{\psi}_{\ell,\alpha}
  = \sup_{v \in \mathcal{V}}
    \bigl[\hat{\psi}_\ell(v) - \kappa_\alpha\,s_\ell(v)\bigr]
\end{equation}
\end{minipage}\hfill
\begin{minipage}{0.48\linewidth}
\begin{equation}
\hat{\psi}_{u,\alpha}
  = \inf_{v \in \mathcal{V}}
    \bigl[\hat{\psi}_u(v) + \kappa_\alpha\,s_u(v)\bigr]
\end{equation}
\end{minipage}
where $\kappa_\alpha$ is a critical value chosen to ensure coverage at level
$1-\alpha$.
Intuitively, the penalty $\kappa_\alpha\,s_\ell(v)$ (resp.\ $\kappa_\alpha\,s_u(v)$)
downweights regions of $\mathcal{V}$ where the estimator is imprecise,
preventing the supremum (infimum) from concentrating on points with
favourable but unreliable estimation error.
The theoretical choice of $\kappa_\alpha$ is based on the standardized error
process.
Focusing on the upper bound (the lower bound is analogous), define
\begin{equation}
T_n(v)
  = \frac{\hat{\psi}_u(v) - \psi_u(v)}{s_u(v)}.
\end{equation}
Under regularity conditions, \citeauthor{chernozhukov2013intersection} show
that $T_n(v)$ can be strongly approximated by a Gaussian process
$T_n^\star(v)$ satisfying
$\sup_{v \in \mathcal{V}}
  \bigl|T_n(v) - T_n^\star(v)\bigr|
  \xrightarrow{p} 0$.
This approximation allows $\kappa_\alpha$ to be chosen as the
$(1-\alpha)$-quantile of $\sup_{v \in \mathcal{V}} T_n^\star(v)$.
In practice, one estimates the covariance structure of $T_n^\star$ from the
data, simulates $B$ draws from the approximating Gaussian process, computes
the simulated suprema, and takes their empirical $(1-\alpha)$-quantile as
$\kappa_\alpha$.

\paragraph{Coverage guarantee.}
Let $\psi_u^\star = \inf_{v \in \mathcal{V}} \psi_u(v)$ denote the true
upper bound.
The precision-corrected estimator $\hat{\psi}_{u,\alpha}$ satisfies
\begin{equation}
P\bigl(\psi_u^\star \leq \hat{\psi}_{u,\alpha}\bigr)
  \geq 1 - \alpha + o(1).
\end{equation}
By construction, $\kappa_\alpha$ is the $(1-\alpha)$ - quantile of
$\sup_{v \in \mathcal{V}} T_n^\star(v)$, and the Gaussian approximation
implies
$P(\sup_{v} T_n(v) > \kappa_\alpha) \approx P(\sup_{v} T_n^\star(v) > \kappa_\alpha)
  = \alpha$.
Hence the interval
$[\hat{\psi}_{\ell,\alpha}, \hat{\psi}_{u,\alpha}]$ achieves asymptotically
valid coverage for the identified set of $\psi^\star$ at level $1-\alpha$.

\paragraph{Application to PNS bounds.}
The Tian--Pearl bounds on $\mathrm{PNS}(z)$ take the form of a maximum over
four lower bound terms and a minimum over four upper bound terms:
\begin{minipage}{0.48\linewidth}
\begin{equation}
L(z)
  = \max_{j \in \{1,\ldots,4\}} \ell_z(j)
\end{equation}
\end{minipage}\hfill
\begin{minipage}{0.48\linewidth}
\begin{equation}
U(z)
  = \min_{k \in \{1,\ldots,4\}} u_z(k),
\end{equation}
\end{minipage}

\noindent
where each $\ell_z(j)$ and $u_z(k)$ is a known function of the observational
and interventional probabilities at covariate value $z$.
This is a finite dimensional instance of the intersection bounds problem.
The precision correction serves two roles in this application.
As an estimator, it removes the systematic bias induced by the max and
min operators selecting components with favourable noise; this debiasing
holds for any estimator with pointwise standard errors.
As an inference procedure, it provides simultaneous coverage of the
bounding functionals which, combined with the algebraic validity of the
Tian--Pearl inequalities at the true distribution, yields coverage of
$\mathrm{PNS}(z)$.
This second guarantee requires the regularity conditions R1--R3
established in Appendix~\ref{sec:pns-theory}, where we show that
M-estimation regularity on the heads of our anchored network is
sufficient (Proposition~\ref{prop:uniform-mest},
Corollary~\ref{cor:pns-parametric},
Theorem~\ref{thm:uniform-pns}).
To our knowledge, this provides the first formal coverage guarantee
for PNS bounds using finite-sample estimation.

To apply the precision corrected framework, we require estimates of the standard errors $s_{\ell_z}(j)$ and $s_{u_z}(m)$, as well as the critical values $\kappa_\alpha$ which depend on the joint distribution of the underlying influence-function process.
The classical approach estimates standard errors using the asymptotic variance of the estimator. This variance is obtained via an influence-function representation that requires Hessian inversion. Critical values are then computed using a multiplier bootstrap over the resulting influence-function process.
For a network with $p$ parameters and $n$ samples, this incurs a cost
of $O(np^2 + p^3)$ per covariate value. In high-dimensional causal estimation where $p \gg 10^5$, this is computationally intractable.
To make the precision correction practical in this regime, we require an alternate approach for estimating $s_{\ell_z}(j)$ and $s_{u_z}(m)$ that scales favourably with model size while capturing epistemic uncertainty arising from finite data. The next section presents such a method.

\vspace{-1em} % Adjust the negative value as needed
\subsubsection{Epistemic Neural Networks for Precision-Corrected Inference}

To avoid the prohibitive computational cost of Hessian-based influence-function
methods, we adopt an alternative framework based on Epistemic Neural Networks
(ENN) \citep{osband2023epistemic}.
Our use of ENNs is motivated by a frequentist perspective.
Rather than interpreting epistemic variability as a Bayesian posterior,
we use the ENN as an amortized mechanism for generating randomized
perturbations of a fitted estimator, in the same spirit as classical
multiplier bootstrap procedures.
This perspective emphasizes the role of ENNs as a scalable approximation
to resampling-based inference, rather than as a model of belief uncertainty.
In the following sections, we describe how the ENN architecture can be
structured to produce coherent perturbations of the empirical estimator,
providing a practical substitute for explicit bootstrap reweighting in
high-dimensional settings.

The ENN framework augments a base network with an epinet and an explicit epistemic index.
This construction is agnostic to the internal architecture of the base network and allows us to sample from our constrained multi-head model without architectural changes.
In our setting, ENNs are used to generate stochastic perturbations of a fitted estimator that approximate finite-sample variability \citep{osband2018randomized,osband2023epistemic}.
These properties make ENNs a natural choice for our precision-corrected PNS bounds.

A standard neural network defines a function
$f_\theta \colon \mathcal{Z} \to \mathbb{R}^k$.
An ENN extends this to a map
$f_\theta \colon \mathcal{Z} \times \mathcal{E} \to \mathbb{R}^k$, where
$\mathcal{E}$ is the space of epistemic indices and $P_\zeta$ is a reference
distribution on $\mathcal{E}$.
For fixed $z$, the distribution of $f_\theta(z,\zeta)$ under
$\zeta \sim P_\zeta$ represents the model's epistemic uncertainty about the
prediction at $z$.
When we consider a sequence of inputs $z_1,\ldots,z_\tau$, the ENN defines a
joint predictive distribution
\begin{equation}
\hat{P}_{1:\tau}(y_{1:\tau})
  = \int_{\mathcal{E}}
      \prod_{t=1}^\tau P\bigl(y_t \mid f_\theta(z_t,\zeta)\bigr)
    \, dP_\zeta(\zeta).
\end{equation}
Dependence of the predictions across $t$ through a shared $\zeta$ encodes
reducible uncertainty that is coherent across inputs.
In our setting the ENN outputs the six probability estimates that define the
Tian--Pearl bounds.
For each $(z,\zeta)$ we obtain
\[
f_\theta(z,\zeta)
=
\bigl(
p_{xy}(z;\theta,\zeta)_{x,y\in\{0,1\}},
\;
\mu_0(z;\theta,\zeta),
\;
\mu_1(z;\theta,\zeta)
\bigr).
\]
and the anchored architecture ensures that these probabilities satisfy the
algebraic constraints for every $\zeta$.
For a fixed $z$ we draw $\zeta_1,\ldots,\zeta_M \sim P_\zeta$.
For each $m$ we compute the four lower bound terms $\ell_z^{(m)}(j)$ and four
upper bound terms $u_z^{(m)}(k)$ from the corresponding probabilities.
We approximate the mean and standard deviation of each bound term under the
epistemic distribution by Monte Carlo averages:
\begin{equation}
\resizebox{\columnwidth}{!}{%
$\begin{aligned}
\bar{\ell}_z(j)
  &= \frac{1}{M}
     \sum_{m=1}^M \ell_z^{(m)}(j),
&
s_{\ell_z}(j)
  &= \sqrt{
        \frac{1}{M-1}
        \sum_{m=1}^M
          \bigl(\ell_z^{(m)}(j) - \bar{\ell}_z(j)\bigr)^2
      },
\\
\bar{u}_z(k)
  &= \frac{1}{M}
     \sum_{m=1}^M u_z^{(m)}(k),
&
s_{u_z}(k)
  &= \sqrt{
        \frac{1}{M-1}
        \sum_{m=1}^M
          \bigl(u_z^{(m)}(k) - \bar{u}_z(k)\bigr)^2
      }.
\end{aligned}$%
}
\end{equation}
The quantities $\bar{\ell}_z(j)$ and $\bar{u}_z(k)$ play the role of
$\hat{\psi}_\ell(v)$ and $\hat{\psi}_u(v)$ in the generic intersection
bounds notation.
The quantities $s_{\ell_z}(j)$ and $s_{u_z}(k)$ are ENN based analogues of
the standard errors $s_\ell(v)$ and $s_u(v)$. Note that compared to the $O(np^2 + p^3)$ Hessian operations required by influence function computation, ENNs require only $O(M)$ forward passes, where $M$ is the number of epistemic index samples at test time.

\paragraph{Critical value computation.}
The precision corrected framework also requires a joint critical value that
respects the dependence structure across the bound terms.
In \cite{chernozhukov2013intersection} this critical value is derived from
a Gaussian approximation to the studentized estimation process.
In our setting we cannot construct such a process directly, because we do
not have an influence function representation.
Instead we use the empirical distribution of ENN samples as a surrogate.

For each epistemic draw $m$ we compute standardized deviations
{\small
\begin{equation}
  \tilde{W}_z^{(m)}(j) = \frac{\ell_z^{(m)}(j) - \bar{\ell}_z(j)}{s_{\ell_z}(j)},
  \quad
  \tilde{V}_z^{(m)}(m) = \frac{u_z^{(m)}(m) - \bar{u}_z(m)}{s_{u_z}(m)}.
\end{equation}
}
We then form the maxima and minima

{
\small
\begin{equation}
  \tilde{W}_z^{L,(m)} = \max_{j \in \{1,\ldots,4\}} \tilde{W}_z^{(m)}(j),
  \quad
  \tilde{W}_z^{U,(m)} = \max_{m \in \{1,\ldots,4\}} \tilde{V}_z^{(m)}(m).
\end{equation}
}

The critical value $\kappa_\alpha^L$ is chosen as the empirical
$(1-\alpha)$ quantile of the sample
$\{\tilde{W}_z^{L,(m)}\}_{m=1}^M$.
The critical value $\kappa_\alpha^U$ is defined analogously from
$\{\tilde{W}_z^{U,(m)}\}_{m=1}^M$.
This construction mirrors the Gaussian process simulation step in
\cite{chernozhukov2013intersection}, but uses the ENN induced empirical
distribution instead of an analytic Gaussian coupling.

\paragraph{Precision-corrected PNS bounds.}
The precision corrected PNS bounds at covariate value $z$ penalise terms
with high epistemic variability.
They are given by

\begin{align}
    \hat{L}_{\alpha, z}
  &= \max_{j \in \{1,\ldots,4\}}
      \bigl[
        \bar{\ell}_z(j)
        - \kappa_\alpha^L\,s_{\ell_z}(j)
      \bigr]\\
\hat{U}_{\alpha, z}
  &= \min_{m \in \{1,\ldots,4\}}
      \bigl[
        \bar{u}_z(m)
        + \kappa_\alpha^U\,s_{u_z}(m)
      \bigr].    
\end{align}

\noindent
The interval $[\hat{L}_\alpha(z), \hat{U}_\alpha(z)]$ provides a
precision corrected estimate of the identified set for $\mathrm{PNS}(z)$
that adjusts for selection bias induced by the max and min operators.
Complete pseudocode for the proposed procedure is provided in 
Appendix~\ref{app:exp-details} - Algorithm \ref{alg:pns_bounds}

\paragraph{Why ENNs?}
ENNs provide a scalable mechanism to generate the jointly coherent, 
function-level perturbations required for multiplier-bootstrap-style 
calibration. In classical frequentist theory, sampling variability is 
typically approximated via the influence function 
$\psi(x) = H^{-1} \nabla_{\theta} \ell(x,\hat{\theta})$. 
Rather than explicitly linearizing the estimator or inverting 
the Hessian $H$, ENNs directly produce stochastic functional 
realizations that perturb the empirical risk minimizer. 
Constrained by a shared architecture, these realizations enforce 
strict functional coherence across covariates and estimands, 
yielding a practical approximation of finite-sample variability. 
This function-space perspective is grounded in Neural Tangent Kernel 
(NTK) theory, which establishes asymptotically linear, jointly 
Gaussian limits for over-parameterized networks \citep{jacot2018neural}, 
and aligns with coherent representations of predictive uncertainty 
\citep{charpentier2020posterior}. While not formally equivalent 
to the multiplier bootstrap, we posit that ENNs offer a computationally 
efficient surrogate for the joint variability needed to calibrate 
the precision-corrected critical value $k_{\alpha}$, a conjecture 
empirically validated in Section 4.

\vspace{-0.5em} % Adjust the negative value as needed

\section{Experimental Evaluation}\label{sec:experiments}

This section evaluates the finite-sample behavior of machine learning estimators for PNS bounds.
The experiments are designed to answer four questions. First, do
standard plug-in estimators produce probability estimates that satisfy the algebraic
constraints of the Tian–Pearl system, and does constraint violation translate into
degraded PNS bound quality? Second, how much of the coverage gain is attributable
to precision correction alone, and how severe is the extremum bias without it? Third,
does the EpiNets uncertainty estimator produce coverage comparable to the
multiplier bootstrap with exact influence functions, and at what computational cost?
Fourth, how does bound quality vary with the amount of available experimental data?

\subsection{Data-Generating Processes and Protocol}
 We evaluate our approach in two settings. In both, the structural causal model (SCM) is fully specified, 
 but the learner observes only a subset of covariates. This induces hidden confounding while allowing 
 exact computation of ground-truth targets.
The first setting uses the synthetic SCM from \citep{li2022probabilities}, featuring a binary treatment $X$, 
a binary outcome $Y$, and a 20-dimensional binary covariate vector $Z$. The learner observes 15 
dimensions of $Z$; the remaining 5 are hidden confounders.
The second setting uses 200 real-valued covariates from the \citet{ACIC2019} dataset, 
augmented with 5 synthetic dimensions. Treatments and potential outcomes are 
generated from an SCM over all 205 variables. The learner observes only the 
original 200 covariates, leaving the 5 synthetic dimensions hidden.
Because the data-generating SCM is known in both settings, the ground-truth $\text{PNS}(z)$ 
and sharp Tian–Pearl bounds are exact. This permits direct evaluation of bound quality and 
coverage. All results are evaluated on disjoint test splits and averaged over $K$ 
independent replicates. Additional details regarding the data generating processes 
can be found in Appendix~\ref{app:dgp-low-dim}, and \ref{app:dgp-high-dim}.

% Samples from Epistemic Neural Networks propagate uncertainty from estimated primitives to the PNS bounds, and precision correction is applied at the level of the bound construction.

\textbf{Metrics.} For each test point $z_i$, the evaluated method outputs an interval $[\ell_i, u_i]$. 
Given the true $\mathrm{PNS}(z_i)$ and oracle identification region $[L(z_i), U(z_i)]$, 
we report: \textit{Constraint violation rate}, the fraction of points where estimated probabilities 
violate the constraints given in Eq.~\ref{eq:observational-constraints}; \textit{Point coverage} 
and \textit{Identified set coverage}, the fractions of points where 
$\mathrm{PNS}(z_i) \in [\ell_i, u_i]$ and $[L(z_i), U(z_i)] \subseteq [\ell_i, u_i]$ respectively, 
both targeting a 95\% nominal level; \textit{Bound width}, the average length $u_i - \ell_i$; 
\textit{Interval score} ~\cite{winkler1972decision}, a proper scoring rule penalising mis-coverage and width. Full definitions are in Appendix~\ref{app:metrics}.
\vspace{-1em} % Adjust the negative value as needed
\subsection{Methods}

We compare four estimators structured as an ablation: each adds one component of the proposed method, allowing the contribution of constraint enforcement and precision correction to be isolated.

\begin{table*}[htbp] % The asterisk (*) makes it span both columns
\caption{Results on low-dimensional (top) and high-dimensional (bottom) datasets.}
\label{tab:results}
\vskip 0.1in
\begin{center}
\begin{small}
\begin{sc}
    \begin{tabular}{lccccc}
    \toprule
    Method & \% Valid (\faThumbsUp\,\faArrowUp) & Pt Cov.\ (\faThumbsUp\,\faArrowUp) & ID cov (\faThumbsUp\,\faArrowUp) & Width (\faThumbsDown\,\faArrowDown) & Int.\ Score (\faThumbsDown\,\faArrowDown) \\
    \midrule
    \multicolumn{6}{c}{\textit{Low-dimensional synthetic}} \\
    \midrule
    S-learner      & 76.29 {\tiny [75.81, 76.78] }  & 74.52 {\tiny [73.90, 75.13] }    & 22.49 {\tiny [22.21, 22.77] } & 0.1730 {\tiny [0.1716, 0.1745] } & 0.7295 {\tiny [0.7132, 0.7458] } \\
    T-learner      & 74.63 {\tiny [74.16, 75.11] }  & 72.39 {\tiny [71.82, 72.95] }    & 27.46 {\tiny [27.22, 27.70] } & 0.1782 {\tiny [0.1772, 0.1793] } & 0.9874 {\tiny [0.9695, 1.0054] } \\
    Anchored       & 100                            & 82.34 {\tiny [82.14, 82.54] }    & 36.36 {\tiny [36.11, 36.61] } & 0.2041 {\tiny [0.2001, 0.2081] } & 0.8102 {\tiny [0.8062, 0.8142] } \\
    MB Full        & 100                            & 98.50 {\tiny [98.10, 98.90] }    & 97.40 {\tiny [97.0, 97.8] }   & 0.3695 {\tiny [0.3665, 0.3725] } & 0.6385 {\tiny [0.6365, 0.6405] } \\
    MB Last Lyr.   & 100                            & 88.10 {\tiny [87.70, 88.30] }    & 73.41 {\tiny [73.38, 73.44] } & 0.2045 {\tiny [0.2015, 0.2075] } & 0.5112 {\tiny [0.5072, 0.5152] } \\
    Anch.\ EpiNet. & 100                            & \textbf{98.19} {\tiny [97.99, 98.59] }    & \textbf{95.59} {\tiny [95.30, 95.80] } & \textbf{0.3303} {\tiny [0.3263, 0.3323] } & \textbf{0.3471} {\tiny [0.3451, 0.3501] } \\
    \midrule
    \multicolumn{6}{c}{\textit{High-dimensional ACIC}} \\
    \midrule
    Anchored       & 100                            & 85.80 {\tiny [85.7, 69.1] }      & 21.50 {\tiny [20.4, 22.6] }   & 0.1630 {\tiny [0.160, 0.165] }   & 0.7815 {\tiny [0.7820, 0.7810] } \\
    MB Last Lyr.   & 100                            & 87.50 {\tiny [85.5, 89.5] }      & 35.40 {\tiny [33.4, 37.4] }   & 0.2010 {\tiny [0.180, 0.220] }   & 0.7200 {\tiny [0.7000, 0.7400] } \\
    Anch.\ EpiNet. & 100                            & \textbf{97.30} {\tiny [96.80, 97.70] } & \textbf{86.00} {\tiny [85.3, 86.7] }   & \textbf{0.3060} {\tiny [0.301, 0.310] }   & \textbf{0.3630} {\tiny [0.3520, 0.3740] } \\
    \bottomrule
    \end{tabular}
\end{sc}
\end{small}
\end{center}
\vskip -0.1in
\end{table*}

\textbf{S-learner and T-learner plug-in.} These are standard baselines that estimate the quantities entering the Tian--Pearl bounds independently, without any structural coordination. The S-learner fits a single model for $P(Y=1 \mid X, Z)$ and obtains $\hat{\mu}_1(z)$ and $\hat{\mu}_0(z)$ by evaluating it at $X=1$ and $X=0$ respectively. The T-learner fits separate models on the treated and control arms to estimate $\mu_1(z)$ and $\mu_0(z)$ directly. In both cases, the joint observational probabilities $p_{xy}(z) = P(X=x, Y=y \mid Z=z)$ are estimated by a separate multinomial model trained on the observational data. All estimates are substituted directly into the Tian--Pearl formulas with no constraint enforcement and no precision correction.

\textbf{Anchored network (no CLR correction).} The anchored architecture introduced in Section 3.1 jointly estimates all six atomic probabilities $(\mu_1, \mu_0, p_{11}, p_{10}, p_{01}, p_{00})$ and enforces the Tian--Pearl consistency constraints by construction via the reparameterisation in Eq.~(10)--(11). Point estimates are plugged into the bound formulas to obtain interval endpoints. No precision correction is applied. This baseline isolates the contribution of constraint enforcement from that of the CLR correction: any improvement over the plug-in methods is due to the anchored architecture alone, and any remaining gap to the full method is attributable to the absence of precision correction.

\textbf{Anchored + ENN + CLR (full method) MB Full.} The complete proposed method as described in Section 3: the anchored architecture with an ENN for epistemic uncertainty quantification, combined with precision-corrected inference via the CLR construction.

\textbf{Last-layer influence function bootstrap (MB Last Lyr).} For the high-dimensional comparison in Section 4.7, we additionally include a frequentist uncertainty estimator that approximates the influence function by differentiating only the final linear layer of the anchored network, reducing the Hessian solve to a $p_{\mathrm{last}}$-dimensional problem. Standard errors derived from a multiplier bootstrap over these approximate influence functions are used in place of ENN-based epistemic uncertainty, with the same CLR precision correction applied. Full network influence function computation is not run in this setting due to the $O(np^2 + p^3)$ cost of Hessian inversion, which is prohibitive at the network sizes required for the ACIC covariate structure.

Additional details regarding the experimental setup can be found
in Appendix~\ref{app:exp-details}.
\vspace{-1em} % Adjust the negative value as needed
\subsection{Results}

\begin{figure}[htbp]
\centering
\includegraphics[width=\columnwidth]{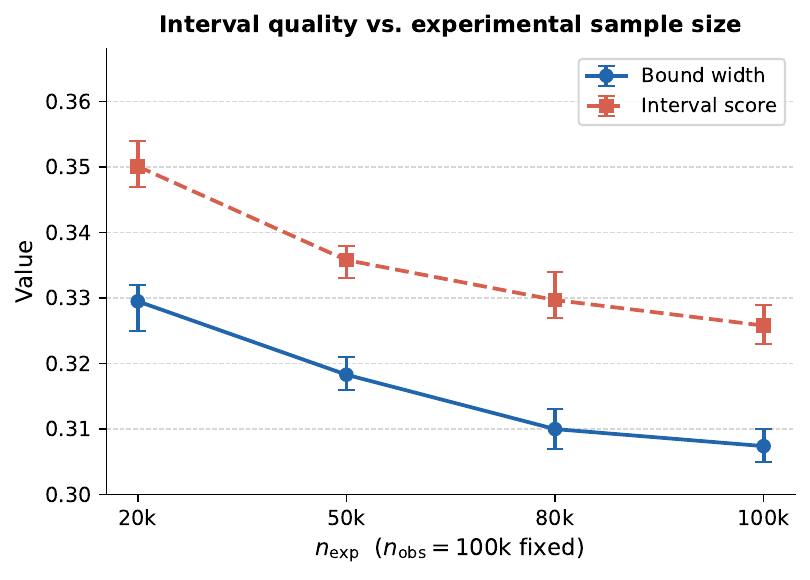}
\caption{Bound width and interval score vs.\ experimental sample size
($n_\mathrm{obs} = 100\mathrm{k}$ fixed). Both decrease monotonically as
$n_\mathrm{exp}$ grows, with the gap between them remaining stable.}
\label{fig:sens_int_nexp}
\end{figure}

\begin{figure}[htbp]
\centering
\includegraphics[width=\columnwidth]{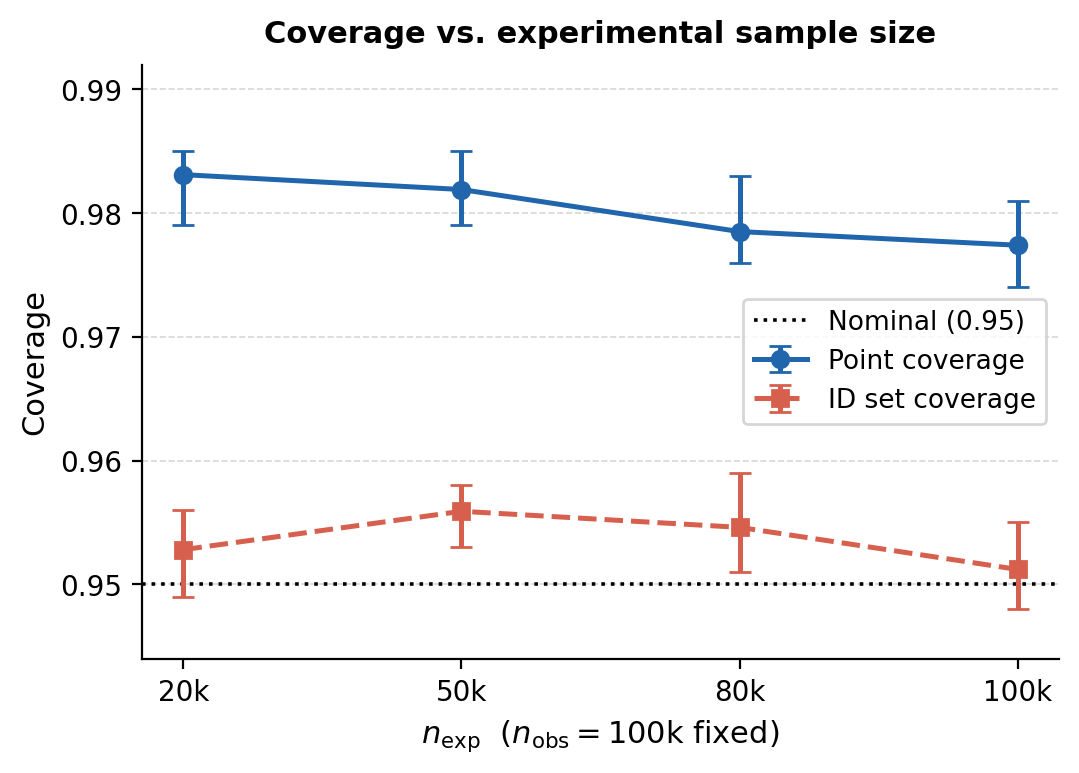}
\caption{Coverage vs.\ joint sample size ($n_\mathrm{obs} = n_\mathrm{exp}$).
ID set coverage remains near nominal across all balanced regimes.}
\label{fig:sens_cov_balanced}
\end{figure}

\vspace{-1em} % Adjust the negative value as needed

Table \ref{tab:results} reports results for the low-dimensional Li et al. benchmark ($n_\mathrm{obs}$ = 100k, $n_\mathrm{exp}$ = 50k) and the high-dimensional ACIC setting ($n_\mathrm{obs}$ = 2M, $n_\mathrm{exp}$ = 1M). Bracketed values indicate 95\% confidence intervals over 1000 independent replicates. Replicates for the low dimensional setting use fresh samples from the exact SCM, and the ACIC replicates are drawn by subsampling a fixed super-population of $3$M observational and $3$M experimental records. In the low dimensional setting, S- and T-learner plug-ins achieve point coverage of $0.745$ and $0.724$ (interval scores $0.730$, $0.987$); their narrow widths reflect infeasible probability configurations rather than sharpness. The anchored network eliminates constraint violations and improves coverage to 0.823, but its interval score (0.810) confirms constraint enforcement alone is insufficient. Causal EpiNet achieves the best performance: valid point coverage at $0.982$ $[0.980, 0.986]$, nominal identified set coverage at $0.956 \ [0.953, 0.958]$, and a score of 0.347. The full-network bootstrap achieves comparable point coverage (0.985) but produces conservative, overly wide intervals (width 0.370 vs. 0.330), resulting in a substantially worse score $(0.639)$. The last-layer bootstrap undercovers at $0.881 \  [0.877, 0.883]$, consistent with uncaptured representation uncertainty. In the high-dimensional ACIC setting, computing the full-network influence function is computationally intractable; it is therefore omitted. Causal EpiNet uniquely maintains valid point coverage $(0.973 \ [0.968, 0.977])$ with a score of $0.363$. Its identified set coverage drops to $0.860$, isolating residual high-dimensional bias in estimating $p_{xy}(z)$ and $\mu_x(z)$ rather than precision miscalibration. The uncorrected estimator undercovers at $0.858 [0.857, 0.891]$ score $(0.782)$. The last-layer bootstrap undercovers at $0.875$ $[0.855, 0.895]$ (score $0.720$); computing influence functions only through the final linear layer produces an incomplete estimate of the asymptotic variance of $\hat{p}_{xy}(z)$ and $\hat{\mu}_x(z)$, which systematically underestimates $s_\ell(j)$ and $s_u(k)$.

We evaluate Causal EpiNet across $(n_{\mathrm{obs}}, n_{\mathrm{exp}})$ regimes to assess behavior under the practically common condition of scarce experimental data. Figures~\ref{fig:sens_cov_balanced} and~\ref{fig:sens_int_nexp} show that as $n_{\mathrm{exp}}$ increases from 20k to 100k with $n_{\mathrm{obs}} = 100\text{k}$ fixed, identified set coverage remains nominal throughout while bound width decreases monotonically. This confirms that ENN uncertainty estimates correctly track experimental data availability without sacrificing calibration. Full results, including balanced scaling up to 1M/1M, are reported in Appendix~\ref{app:sensitivity}. To assess computational efficiency, we record the training and inference wall-clock times for all evaluated methods. Because they require only a single forward pass rather than repeated model retraining or costly influence function evaluations, ENNs incur substantially lower computational overhead than the bootstrap baselines. A comprehensive runtime breakdown is deferred to the Appendix~\ref{app:compute}.

Note that the point coverage intrinsically exceeds identified (ID) set coverage. Because $\mathrm{PNS}(z) \in [L(z), U(z)]$ deterministically, $[L(z_i), U(z_i)] \subseteq [\ell_i, u_i]$ implies $\mathrm{PNS}(z_i) \in [\ell_i, u_i]$, but not conversely; an interval can miss an oracle endpoint yet still cover the true parameter. The CLR correction specifically targets point coverage. Thus, nominal ID set coverage in the low-dimensional setting reflects precise first-stage estimation rather than inferential conservatism. In the ACIC benchmark, residual bias in $\hat{p}_{xy}(z)$ and $\hat{\mu}_x(z)$ degrades ID set coverage to $0.860$, but point coverage remains valid at $0.973$ because the true $\mathrm{PNS}(z)$ lies strictly interior to the estimated bounds. Sub-nominal ID set coverage therefore diagnoses first-stage estimation error, not inferential failure.

\vspace{-0.5em}
\section{Discussion and Conclusion}
\vspace{-0.5em}
We have presented Causal EpiNets, a framework for finite-sample inference on
individual-level causal effects via PNS bounds. By combining a constrained anchored
architecture with ENN-based uncertainty quantification and CLR precision-corrected
inference, the method achieves near-nominal coverage without requiring unconfoundedness
or positivity on the observational regime — assumptions that are routinely violated in
practice and that methods based on CATE estimation cannot dispense with. The only
requirements are access to experimental and observational data from the same
population and the stability assumption of Section~2. The framework extends directly
to the probability of harm $P(Y_{x'} = 1, Y_x = 0 \mid Z = z)$ and, together with
PNS, provides a complete characterisation of individual-level treatment effect
heterogeneity in the binary setting. We view this as a strong baseline for a line of
work that has lacked practical estimation methods despite a rich 20-year-old
identification theory.

\textbf{Limitations.} The current framework is restricted to binary treatments 
and outcomes. While theoretical progress exists for multi-valued and continuous 
treatments~\citep{kawakami24a, li2024probabilities_nonbinary}, developing a 
practical estimation procedure remains an open problem. 
Furthermore, scalability to high-dimensional unstructured covariates, such as 
text or images, requires further investigation. In these settings, representation 
learning becomes the dominant challenge. Although our anchored architecture 
enforces consistency constraints over a learned embedding, the generalization of 
these embeddings and the behavior of ENN epistemic uncertainty over them require 
future study.

\textbf{Future work.} A primary obstacle to progress in this domain is the absence of 
large-scale benchmark datasets combining experimental and observational data. This 
data scarcity restricts empirical evaluation to synthetic and semi-synthetic environments. 
We encourage the development of public benchmarks to facilitate the rigorous evaluation 
of individual-level causal inference methods. Methodologically, natural next steps 
include extending identification and estimation theory beyond the binary setting 
and incorporating modern representation learning for multi-modal and document-level 
covariates into the anchored ENN framework.

\newpage
\bibliography{uai2026-template}

@article{tian2000probabilities,
  author    = {Tian, Jin and Pearl, Judea},
  title     = {Probabilities of Causation: Bounds and Identification},
  journal   = {Annals of Mathematics and Artificial Intelligence},
  volume    = {28},
  number    = {1--4},
  pages     = {287--313},
  year      = {2000},
  doi       = {10.1023/A:1018912507879},
  note      = {\url{https://ftp.cs.ucla.edu/pub/stat_ser/r271-A.pdf}}
}

@article{chernozhukov2013intersection,
  author    = {Chernozhukov, Victor and Lee, Sokbae and Rosen, Adam M.},
  title     = {Intersection Bounds: Estimation and Inference},
  journal   = {Econometrica},
  volume    = {81},
  number    = {2},
  pages     = {667--737},
  year      = {2013},
  doi       = {10.3982/ECTA8718}
}

@article{chernozhukov2013gaussian,
  author    = {Chernozhukov, Victor and Chetverikov, Denis and Kato, Kengo},
  title     = {Gaussian Approximations and Multiplier Bootstrap for Maxima of Sums of High-Dimensional Random Vectors},
  journal   = {Annals of Statistics},
  volume    = {41},
  number    = {6},
  pages     = {2786--2819},
  year      = {2013},
  doi       = {10.1214/13-AOS1161}
}

@book{vandervaart1998asymptotic,
  author    = {{van der Vaart}, A. W.},
  title     = {Asymptotic Statistics},
  series    = {Cambridge Series in Statistical and Probabilistic Mathematics},
  publisher = {Cambridge University Press},
  address   = {Cambridge},
  year      = {1998},
  doi       = {10.1017/CBO9780511802256}
}

@book{vandervaart1996weak,
  author    = {{van der Vaart}, A. W. and Wellner, Jon A.},
  title     = {Weak Convergence and Empirical Processes: With Applications to Statistics},
  series    = {Springer Series in Statistics},
  publisher = {Springer},
  address   = {New York},
  year      = {1996},
  doi       = {10.1007/978-1-4757-2545-2}
}

@book{bickel1993efficient,
  author    = {Bickel, Peter J. and Klaassen, Chris A. J. and Ritov, Ya'acov and Wellner, Jon A.},
  title     = {Efficient and Adaptive Estimation for Semiparametric Models},
  series    = {Johns Hopkins Series in the Mathematical Sciences},
  publisher = {Johns Hopkins University Press},
  address   = {Baltimore},
  year      = {1993}
}

@incollection{newey1994largesample,
  author    = {Newey, Whitney K. and McFadden, Daniel},
  title     = {Large Sample Estimation and Hypothesis Testing},
  booktitle = {Handbook of Econometrics},
  editor    = {Engle, Robert F. and McFadden, Daniel L.},
  volume    = {4},
  chapter   = {36},
  pages     = {2111--2245},
  publisher = {Elsevier},
  year      = {1994},
  doi       = {10.1016/S1573-4412(05)80005-4}
}

@article{white1989JASA,
  author    = {White, Halbert},
  title     = {Some Asymptotic Results for Learning in Single Hidden-Layer Feedforward Network Models},
  journal   = {Journal of the American Statistical Association},
  volume    = {84},
  number    = {408},
  pages     = {1003--1013},
  year      = {1989},
  doi       = {10.1080/01621459.1989.10478865}
}

@article{chenwhite1999IT,
  author    = {Chen, Xiaohong and White, Halbert L.},
  title     = {Improved Rates and Asymptotic Normality for Nonparametric Neural Network Estimators},
  journal   = {IEEE Transactions on Information Theory},
  volume    = {45},
  number    = {2},
  pages     = {682--691},
  year      = {1999},
  doi       = {10.1109/18.749011}
}

@article{shen2023sieve,
  author    = {Shen, Xiaoxi and Jiang, Chao and Sakhanenko, Lyudmila and Lu, Qing},
  title     = {Asymptotic Properties of Neural Network Sieve Estimators},
  journal   = {Journal of Nonparametric Statistics},
  volume    = {35},
  number    = {4},
  pages     = {839--868},
  year      = {2023},
  doi       = {10.1080/10485252.2023.2209218}
}

@article{white1982mlmisspecified,
  author    = {White, Halbert},
  title     = {Maximum Likelihood Estimation of Misspecified Models},
  journal   = {Econometrica},
  volume    = {50},
  number    = {1},
  pages     = {1--25},
  year      = {1982},
  doi       = {10.2307/1912526}
}

@book{mccullagh1989glm,
  author    = {McCullagh, Peter and Nelder, John A.},
  title     = {Generalized Linear Models},
  edition   = {2nd},
  series    = {Monographs on Statistics and Applied Probability},
  publisher = {Chapman and Hall},
  address   = {London},
  year      = {1989},
  doi       = {10.1007/978-1-4899-3242-6}
}

@inproceedings{kohliang2017,
  author    = {Koh, Pang Wei and Liang, Percy},
  title     = {Understanding Black-box Predictions via Influence Functions},
  booktitle = {Proceedings of the 34th International Conference on Machine Learning},
  pages     = {1885--1894},
  year      = {2017}
}

@inproceedings{martens2010,
  author    = {Martens, James},
  title     = {Deep Learning via {H}essian-free Optimization},
  booktitle = {Proceedings of the 27th International Conference on Machine Learning},
  pages     = {735--742},
  year      = {2010}
}

@article{pearlmutter1994,
  author    = {Pearlmutter, Barak A.},
  title     = {Fast Exact Multiplication by the {H}essian},
  journal   = {Neural Computation},
  volume    = {6},
  number    = {1},
  pages     = {147--160},
  year      = {1994},
  doi       = {10.1162/neco.1994.6.1.147}
}

@article{frauen2024neural,
  title     = {A Neural Framework for Generalized Causal Sensitivity Analysis},
  author    = {Frauen, Dennis and Imrie, Fergus and Curth, Alicia and Melnychuk, 
               Valentyn and Feuerriegel, Stefan and van der Schaar, Mihaela},
  journal   = {International Conference on Learning Representations (ICLR)},
  year      = {2024}
}

@article{chen2024doubly,
  title     = {Doubly Robust Causal Effect Estimation under Networked Interference via Targeted Learning},
  author    = {Chen, Weilin and Cai, Ruichu and Yang, Zeqin and Qiao, Jie and Yan, Yuguang and Li, Zijian and Hao, Zhifeng},
  journal   = {International Conference on Machine Learning (ICML)},
  year      = {2024}
}

@article{blanc2020implicit,
  title     = {Implicit regularization for deep neural networks driven by an Ornstein-Uhlenbeck like process},
  author    = {Blanc, Guy and Gupta, Neha and Valiant, Gregory and Valiant, Paul},
  journal   = {Conference on Learning Theory (COLT)},
  year      = {2020}
}

@article{li2018algorithmic,
  title     = {Algorithmic Regularization in Over-parameterized Matrix Sensing and Neural Networks with Quadratic Activations},
  author    = {Li, Yuanzhi and Ma, Tengyu and Zhang, Hongyang},
  journal   = {Conference on Learning Theory (COLT)},
  year      = {2018}
}

@article{Kallus2022Harm,
  title     = {What's the Harm? Sharp Bounds on the Fraction Negatively Affected by Treatment},
  author    = {Nathan Kallus},
  journal   = {arXiv preprint arXiv:2205.10327},
  year      = {2022},
  doi       = {10.48550/arXiv.2205.10327},
  url       = {https://arxiv.org/abs/2205.10327}
}

@article{MuellerPearl2023,
  title     = {Personalized decision making – A conceptual introduction},
  author    = {Mueller, Scott and Pearl, Judea},
  journal   = {Journal of Causal Inference},
  volume    = {11},
  number    = {1},
  pages     = {13},
  year      = {2023},
  publisher = {De Gruyter}
}

@article{li2022probabilities,
  title     = {Probabilities of Causation: Adequate Size of Experimental and Observational Samples},
  author    = {Li, Ang and Mao, Ruirui and Pearl, Judea},
  journal   = {arXiv preprint arXiv:2210.05027},
  year      = {2022},
  doi       = {10.48550/arXiv.2210.05027}
}

@misc{ACIC2019,
  author       = {{ACIC}},
  title        = {ACIC 2019 Data Challenge Datasets},
  year         = {2019},
  howpublished = {\url{https://sites.google.com/view/acic2019datachallenge/data-challenge}}
}

@inproceedings{Osband2023Epistemic,
  title     = {Epistemic Neural Networks},
  author    = {Ian Osband and Zheng Wen and Seyed Mohammad Asghari and Vikranth Dwaracherla and Morteza Ibrahimi and Xiuyuan Lu and Benjamin Van Roy},
  booktitle = {Advances in Neural Information Processing Systems (NeurIPS 2023)},
  year      = {2023},
  url       = {https://arxiv.org/abs/2107.08924}
}

@inproceedings{osband2018randomized,
  title     = {Randomized Prior Functions for Deep Reinforcement Learning},
  author    = {Osband, Ian and Aslanides, John and Cassirer, Albin},
  booktitle = {Advances in Neural Information Processing Systems},
  volume    = {31},
  year      = {2018}
}

@inproceedings{jacot2018neural,
  title        = {Neural Tangent Kernel: Convergence and Generalization in Neural Networks},
  author       = {Jacot, Arthur and Gabriel, Franck and Hongler, Clément},
  booktitle    = {Advances in Neural Information Processing Systems},
  year         = {2018},
  pages        = {8580--8589},
  organization = {Curran Associates, Inc.}
}

@incollection{charpentier2020posterior,
  title     = {Posterior Network: Uncertainty Estimation without OOD Samples via Density-Based Pseudo-Counts},
  author    = {Charpentier, Bertrand and Z{\"u}gner, Daniel and G{\"u}nnemann, Stephan},
  booktitle = {Advances in Neural Information Processing Systems},
  year      = {2020},
  publisher = {Curran Associates, Inc.}
}

@misc{li2022probabilitiescausationadequatesize,
      title={Probabilities of Causation: Adequate Size of Experimental and Observational Samples},
      author={Ang Li and Ruirui Mao and Judea Pearl},
      year={2022},
      eprint={2210.05027},
      archivePrefix={arXiv},
      primaryClass={cs.AI},
      url={https://arxiv.org/abs/2210.05027},
}

@article{winkler1972decision,
  title={A decision-theoretic approach to interval estimation},
  author={Winkler, Robert L},
  journal={Journal of the American Statistical Association},
  volume={67},
  number={337},
  pages={187--191},
  year={1972},
  publisher={Taylor \& Francis}
}

@software{jax2018github,
  author  = {James Bradbury and Roy Frostig and Peter Hawkins and
             Matthew James Johnson and Chris Leary and Dougal Maclaurin and
             George Necula and Adam Paszke and Jake Vander{P}las and
             Skye Wanderman-{M}ilne and Qiao Zhang},
  title   = {{JAX}: composable transformations of {P}ython+{N}um{P}y programs},
  url     = {http://github.com/google/jax},
  version = {0.3.13},
  year    = {2018}
}

@article{kidger2021equinox,
  author  = {Patrick Kidger and Cristian Garcia},
  title   = {{E}quinox: neural networks in {JAX} via callable {P}y{T}rees and filtered transformations},
  journal = {Differentiable Programming workshop at Neural Information Processing Systems 2021},
  year    = {2021}
}

@article{dasgupta2018treatment,
  author    = {Jiannan Lu and Peng Ding and Tirthankar Dasgupta},
  title     = {Treatment effects on ordinal outcomes: Causal estimands and sharp bounds},
  journal   = {Journal of Educational and Behavioral Statistics},
  volume    = {43},
  number    = {5},
  pages     = {540--567},
  year      = {2018}
}

@article{semenova2023debiased,
  author    = {Vira Semenova},
  title     = {Debiased machine learning of set-identified linear models},
  journal   = {Journal of Econometrics},
  volume    = {235},
  pages     = {1725--1746},
  year      = {2023}
}

@inproceedings{alaa2023conformal,
  author    = {Ahmed M. Alaa and Zaid Ahmad and Mark van der Laan},
  title     = {Conformal meta-learners for predictive inference of individual treatment effects},
  booktitle = {Advances in Neural Information Processing Systems},
  volume    = {36},
  year      = {2023}
}

@article{wang2025estimating,
  author    = {Shuai Wang and Ang Li},
  title     = {Estimating probabilities of causation with machine learning models},
  journal   = {arXiv preprint arXiv:2502.08858},
  year      = {2025}
}

@InProceedings{kawakami24a,
  title     = {Probabilities of Causation for Continuous and Vector Variables},
  author    = {Kawakami, Yuta and Kuroki, Manabu and Tian, Jin},
  booktitle = {Proceedings of the Fortieth Conference on Uncertainty in Artificial Intelligence},
  pages     = {1901--1921},
  year      = {2024},
  editor    = {Kiyavash, Negar and Mooij, Joris M.},
  volume    = {244},
  series    = {Proceedings of Machine Learning Research},
  month     = {15--19 Jul},
  publisher = {PMLR},
  url       = {https://proceedings.mlr.press/v244/kawakami24a.html}
}

@InProceedings{li2024probabilities_nonbinary,
  author    = {Li, Ang and Pearl, Judea},
  title     = {Probabilities of Causation with Nonbinary Treatment and Effect},
  booktitle = {Proceedings of the Thirty-Eighth AAAI Conference on Artificial Intelligence},
  series    = {AAAI Technical Tracks},
  volume    = {38},
  number    = {18},
  pages     = {20465--20472},
  year      = {2024},
  month     = {Mar},
  publisher = {Association for the Advancement of Artificial Intelligence (AAAI)},
  doi       = {10.1609/aaai.v38i18.30030},
  url       = {https://ojs.aaai.org/index.php/AAAI/article/view/30030}
}

\newpage
\appendix
\onecolumn

\section{Related Work} \label{app:related-work}

Causal EpiNets sits at the intersection of three research threads: the
partial identification of individual causal quantities, the statistical
theory of inference for envelope-type bounds, and the application of
machine learning to causal uncertainty quantification. We discuss each
in turn, highlighting where prior work falls short of the finite-sample,
high-dimensional estimation problem we address.

\paragraph{Partial identification of individual causal effects.}
The fundamental difficulty motivating this paper is that individual
treatment effects cannot be point-identified from data, even in
randomised experiments. \citet{tian2000probabilities} showed that
PNS---the probability that treatment is both necessary and sufficient
for a positive outcome---admits sharp nonparametric bounds when
experimental and observational data are combined, and that neither data
source alone is sufficient for informative inference. Their derivation
exploits the algebraic constraints linking observational joint
probabilities to interventional success probabilities within a single
structural causal model, the same constraint system our anchored
architecture enforces by construction.
\citet{MuellerPearl2023} extend this perspective to personalised
decision-making, demonstrating that confounding in observational data
actually \emph{narrows} the PNS identified set rather than widening it,
and that combining data sources can collapse bounds to point estimates
for specific subpopulations. Both works assume exact knowledge of the
relevant marginal and joint probabilities; the estimation problem that
arises when these must be learned from finite, high-dimensional data is
left open, and is precisely the problem we address.

Closely related estimands have been studied in adjacent literatures.
\citet{dasgupta2018treatment} derive sharp bounds on the probabilities
that treatment is strictly beneficial or strictly harmful for ordinal
outcomes, using Fréchet-Hoeffding arguments under fixed marginal
distributions of potential outcomes. Their bounds apply directly to
randomised experiments and observational studies under unconfoundedness,
and the closed-form expressions they derive are compatible with flexible
nonparametric potential-outcome models---a feature we share. However,
their focus is on identification rather than estimation, and they do not
address the finite-sample inference problem for bounds computed from
estimated nuisance parameters. \citet{Kallus2022Harm} target the
fraction of units negatively affected by treatment, a population-level
summary of individual harm that is also partially identified, and
develop doubly-robust inference that remains valid even under nuisance
misspecification. Kallus's estimand and ours are closely connected---both
bound a quantity that captures the joint distribution of potential
outcomes---but his inference strategy operates on a scalar summary
across the population, whereas our target is the full covariate-conditional
PNS function, which introduces the high-dimensional estimation challenges
central to our contribution.

\paragraph{Inference for intersection bounds.}
Even when the underlying nuisance functions are estimated consistently,
constructing valid confidence intervals for PNS bounds is non-trivial.
PNS takes the form of a maximum over lower-bound terms and a minimum
over upper-bound terms; substituting estimated quantities into these
operators introduces systematic selection bias that does not vanish as
sample size grows \citep{chernozhukov2013intersection}. This is the
\emph{envelope-induced bias} problem, which is intrinsic to partial
identification and distinct from ordinary estimation bias.
\citet{chernozhukov2013intersection} develop a general solution:
they propose precision-corrected estimators that penalise each bound
component by its estimated standard error and calibrate a joint critical
value via strong approximation of the studentised error process by a
sequence of Gaussian processes. Their coverage guarantees require that
the bounding function estimators admit asymptotically linear
representations with well-behaved influence functions---a condition
satisfied by kernel and series estimators but not, in general, by neural
networks trained by stochastic gradient methods.

\citet{semenova2023debiased} extend this program to high-dimensional
settings, constructing Neyman-orthogonal estimating equations for the
support function of a partially identified set and proving root-$n$
consistency when nuisance parameters are estimated with regularised
machine learning. Her framework handles identified set boundaries that
can be expressed via moment equations, and the multiplier bootstrap she
proposes is closely related to our simulation-based critical value
procedure. The key structural difference is that PNS bounds take the
form of an intersection of inequality constraints rather than a support
function, and the bounding functions depend on multiple probability
estimates that are jointly constrained by a structural causal model.
Adapting Semenova's orthogonality arguments to this setting would
require characterising the influence function of the PNS bound with
respect to all six component probabilities simultaneously, accounting
for the algebraic constraints between them. Our approach sidesteps this
challenge by using Epistemic Neural Networks to represent the joint
epistemic distribution over the component probabilities and applying
the precision-corrected construction of \citet{chernozhukov2013intersection}
in an approximate sense. We do not claim their asymptotic guarantees;
rather, our contribution is to show empirically that this construction
yields near-nominal finite-sample coverage when combined with a
constraint-aware architecture, in regimes where direct application of
the CLR framework is infeasible.

\paragraph{Machine learning for PNS and ITE uncertainty.}
Several recent papers apply machine learning to PNS estimation or to
closely related uncertainty quantification problems, but none jointly
addresses constraint enforcement, envelope bias, and calibrated coverage.
\citet{wang2025estimating} propose training MLPs to predict PNS bound
values from subpopulation features, treating bound estimation as a
supervised regression problem where SCM-derived bounds serve as
regression targets. Their approach is fundamentally different from
ours: it assumes access to a known SCM and sufficient within-subpopulation
data to compute exact bounds for the training set, then generalises to
data-scarce subpopulations by interpolation. It does not enforce the
Tian-Pearl algebraic constraints during estimation, does not account
for envelope-induced bias, and produces point predictions rather than
calibrated confidence intervals. The method provides a useful
feasibility demonstration but does not constitute a valid inference
procedure.

For the more standard ITE setting---where ignorability and positivity
are assumed---\citet{alaa2023conformal} develop conformal meta-learners
that apply conformal prediction on top of pseudo-outcome CATE estimators
to obtain distribution-free predictive intervals for ITEs. They
establish that doubly-robust and IPW meta-learners produce conformity
scores that stochastically dominate oracle scores, ensuring marginal
coverage without distributional assumptions on the outcome model. This
is an elegant solution to uncertainty quantification under point
identification. Our setting is complementary: PNS bounds are the
appropriate target precisely when ignorability and positivity cannot
be assumed, and conformal methods do not apply to partially identified
quantities. The two approaches together suggest a division of labour:
conformal meta-learners for settings where identification assumptions
are plausible, and precision-corrected PNS bounds for settings where
they are not.

\section{Asymptotic validity of neural PNS bounds}
\label{sec:pns-theory}

In this appendix we state the regularity conditions on the neural estimators of
the atomic probabilities, prove a generic coverage theorem for PNS intervals
under these conditions, and then show how standard finite-dimensional
M-estimation assumptions on the neural network heads imply the required
regularity. We also discuss how approximate numerical solutions used in
practice, such as conjugate gradient solves of Hessian systems, affect the
theoretical guarantees. 

Section~\ref{subsec:pointwise} establishes \emph{pointwise} coverage guarantees
for a fixed covariate value $z$. Section~\ref{subsec:uniform} extends these
results to \emph{uniform} coverage over all $z \in \mathcal{Z}$ simultaneously,
which is essential for valid inference on the entire conditional treatment
effect surface.

\subsection{Setup and notation}

Fix a covariate value $z \in \mathcal{Z}$, let $\eta(z) := \bigl(\mu_1(z), \mu_0(z), p_{11}(z), p_{10}(z), p_{01}(z), p_{00}(z) \bigr)^\top$ be the atomic probability vector and let $\hat\eta(z) := \bigl(\hat\mu_1(z), \hat\mu_0(z), \hat p_{11}(z), \hat p_{10}(z), \hat p_{01}(z), \hat p_{00}(z) \bigr)^\top$ be its estimator. Note that both $\eta(z)$ and $\hat{\eta}(z)  \in \mathbb{R}^6$. Next, let $A \in \mathbb{R}^{7\times 6}$ denote the matrix encoding the
\citep[Eq.~(24)]{tian2000probabilities} bounding functionals, such that $g(z) = A\eta(z)$ and $\hat{g}(z) = A\hat{\eta}(z)$, where both $g(z)$ and $\hat{g}(z) \in \mathbb{R}^7$.
We index the three lower-bound terms as $g_1,g_2,g_3$ and the four upper-bound
terms as $g_4,\dots,g_7$. The true PNS satisfies the Tian--Pearl inequalities
\[
  \max\{0,g_1(z),g_2(z),g_3(z)\}
  \;\le\; \mathrm{PNS}(z)
  \;\le\;
  \min\{g_4(z),g_5(z),g_6(z),g_7(z)\}.
\]

Let $W_i = (S_i, Z_i, X_i, Y_i)$ where $S_i \in \{\text{obs}, \text{exp}\}$ is a regime indicator.
We assume $W_1,\dots,W_n$ are i.i.d.\ draws from a mixture distribution: with probability $\pi_{\mathrm{obs}}$ we draw from the observational regime, with probability $\pi_{\mathrm{exp}} = 1 - \pi_{\mathrm{obs}}$ from the experimental regime.
As $n \to \infty$, we assume $n_{\mathrm{obs}}/n \to \pi_{\mathrm{obs}}$ and $n_{\mathrm{exp}}/n \to \pi_{\mathrm{exp}}$ for some fixed $\pi_{\mathrm{obs}} \in (0,1)$.
All moments are finite under both regimes.

%%%%%%%%%%%%%%%%%%%%%%%%%%%%%%%%%%%%%%%%%%%%%%%%%%%%%%%%%%%%%%%%%%%%%%%%%%%%%
\subsection{Pointwise coverage guarantees}
\label{subsec:pointwise}
%%%%%%%%%%%%%%%%%%%%%%%%%%%%%%%%%%%%%%%%%%%%%%%%%%%%%%%%%%%%%%%%%%%%%%%%%%%%%

We first establish coverage guarantees for a fixed covariate value $z \in \mathcal{Z}$.

\subsubsection{Regularity conditions R1--R3}
\label{subsubsec:regularity-R}

The coverage result for the PNS interval relies on three regularity conditions
on the behaviour of $\hat\eta(z)$.

\begin{description}
  \item[R1 (Asymptotic linearity).]
  For each fixed $z$, there exists a $6$-dimensional influence function
  $\phi(W;z)$ with $\mathbb{E}[\phi(W;z)] = 0$ and
  $\mathbb{E}\|\phi(W;z)\|^2 < \infty$ such that
  \begin{equation}
    \label{eq:R1-main}
    \hat\eta(z) - \eta(z)
    = \frac{1}{n}\sum_{i=1}^n \phi(W_i;z) + o_p(n^{-1/2}) .
  \end{equation}

  \item[R2 (Consistent plug-in influence functions).]
  There exist plug-in influence estimates $\hat\phi(W_i;z)$ such that, for each
  fixed $z$,
  \begin{equation}
    \label{eq:R2-main}
    \frac{1}{n}\sum_{i=1}^n
      \bigl\|\hat\phi(W_i;z) - \phi(W_i;z)\bigr\|^2
    \xrightarrow{p} 0 .
  \end{equation}

  \item[R3 (Multiplier bootstrap validity).]
  Let $\varphi(W_i;z) := A\,\phi(W_i;z)$ and
  $\hat\varphi(W_i;z) := A\,\hat\phi(W_i;z)$, and define
  \[
    \hat s_r^2(z)
    := \frac{1}{n}\sum_{i=1}^n \hat\varphi_r(W_i;z)^2,
    \qquad r=1,\dots,7.
  \]
  Set
  \[
    T_{n,r}(z)
      := \frac{\hat g_r(z) - g_r(z)}{\hat s_r(z)},
    \qquad
    M_n(z) := \max_{r\le 7} |T_{n,r}(z)| .
  \]
  Let $M_n^*(z)$ be the Gaussian multiplier bootstrap analogue of $M_n(z)$
  based on $\hat\varphi(W_i;z)$: if $\xi_1,\dots,\xi_n \sim N(0,1)$ are
  i.i.d.\ multipliers, define
  \[
    T_r^*(z)
    := \frac{1}{\sqrt{n}\,\hat s_r(z)}
       \sum_{i=1}^n \xi_i\,\hat\varphi_r(W_i;z),
    \qquad
    M_n^*(z) := \max_{r\le 7} |T_r^*(z)|.
  \]
  Condition R3 requires that, conditional on the data, the distribution of $M_n^*(z)$
  converges in probability to the same non-degenerate limit as that of $M_n(z)$,
  so the bootstrap quantile $\hat \kappa_{1-\alpha}(z)$ consistently estimates the
  $(1-\alpha)$-quantile of $M_n(z)$.  Sufficient conditions for R3 include
  bounded $(2+\delta)$th moments and non-degeneracy of the covariance matrix;
  see, e.g., \citep{vandervaart1996weak,chernozhukov2013gaussian}.
\end{description}

Conditions R1 and R2 are the usual regularity and plug-in requirements for a
finite-dimensional functional in a (semi)parametric model, expressed in
influence-function form; see, for example, \citep{bickel1993efficient}
and \citep{vandervaart1998asymptotic}. Condition R3 is a standard
bootstrap central limit condition for a finite-dimensional, asymptotically
linear estimator \citep{vandervaart1998asymptotic,chernozhukov2013gaussian}.

\subsubsection{Generic coverage theorem under R1--R3}
\label{subsubsec:generic-theorem}

We first give a generic result: if R1--R3 hold for the atomic estimators, then
the precision-corrected PNS interval has asymptotic coverage. Define the PNS interval at $z$ by
\begin{align}
  \underline{\mathrm{PNS}}(z)
  &:= \max\Bigl\{
    0,\
    \max_{r \in \{1,2,3\}}
      \bigl[\hat g_r(z) - \hat \kappa_{1-\alpha}(z)\,\hat s_r(z)\bigr]
  \Bigr\}, \label{eq:pns-lower-corrected}\\
  \overline{\mathrm{PNS}}(z)
  &:= \min_{r \in \{4,5,6,7\}}
      \bigl[\hat g_r(z) + \hat \kappa_{1-\alpha}(z)\,\hat s_r(z)\bigr], \label{eq:pns-upper-corrected}
\end{align}
where $\hat \kappa_{1-\alpha}(z)$ is the empirical $(1-\alpha)$-quantile of
$M_n^*(z)$ conditional on the data.

\begin{theorem}[Asymptotically valid PNS intervals under R1--R3]
\label{thm:pns-R1R3}
Assume the Tian--Pearl inequalities relating $\mathrm{PNS}(z)$ and $g_r(z)$
hold at the true distribution, and that the estimators of the atomic
probabilities satisfy regularity conditions R1--R3. Then, for each fixed
$z \in \mathcal{Z}$,
\[
  \Pr\Bigl(
    \underline{\mathrm{PNS}}(z)
    \le \mathrm{PNS}(z)
    \le \overline{\mathrm{PNS}}(z)
  \Bigr)
  \;\to\; 1 - \alpha
  \quad \text{as } n_{\mathrm{obs}}, n_{\mathrm{exp}} \to \infty.
\]
\end{theorem}

\begin{proof}[Proof sketch]
Under R1, $\hat\eta(z) - \eta(z)$ admits the asymptotic linear expansion
\eqref{eq:R1-main}. By linearity,
\[
  \hat g(z) - g(z)
  = A\bigl(\hat\eta(z) - \eta(z)\bigr)
  = \frac{1}{n}\sum_{i=1}^n \varphi(W_i;z) + o_p(n^{-1/2}),
\]
with $\varphi(W_i;z) := A\,\phi(W_i;z) \in \mathbb{R}^7$. Combined with R2 and
Lemma~\ref{lem:SE-consistency} (below), this yields a joint central limit
theorem for the studentized vector
\[
  \Bigl(
    T_{n,1}(z),\dots,T_{n,7}(z)
  \Bigr)
  =
  \Bigl(
    \tfrac{\hat g_1(z)-g_1(z)}{\hat s_1(z)},\dots,
    \tfrac{\hat g_7(z)-g_7(z)}{\hat s_7(z)}
  \Bigr),
\]
and hence for the max statistic $M_n(z) = \max_{r\le 7} |T_{n,r}(z)|$.

Assumption R3 asserts that the conditional distribution of $M_n^*(z)$ given
the data converges to the same limit as that of $M_n(z)$, so the bootstrap
critical value $\hat \kappa_{1-\alpha}(z)$ consistently estimates the
$(1-\alpha)$-quantile of $M_n(z)$.
On the event $\{|T_{n,r}(z)| \le \hat \kappa_{1-\alpha}(z)\ \forall r\}$ we have
\begin{equation}
\label{eq:coverage-event}
  \hat g_r(z) - \hat \kappa_{1-\alpha}(z)\,\hat s_r(z)
  \le g_r(z)
  \le \hat g_r(z) + \hat \kappa_{1-\alpha}(z)\,\hat s_r(z)
  \quad \forall r.    
\end{equation}
Combining this with the population inequalities
\begin{equation}
\label{eq:tian-pearl-ineq}
  \max\{0,g_1(z),g_2(z),g_3(z)\}
  \le \mathrm{PNS}(z)
  \le \min\{g_4(z),g_5(z),g_6(z),g_7(z)\}  
\end{equation}
shows that, on this event,
$\mathrm{PNS}(z)$ lies within
$[\underline{\mathrm{PNS}}(z),\overline{\mathrm{PNS}}(z)]$. By R3, the
probability of this event converges to $1-\alpha$, which yields the result.
\end{proof}

For completeness, we record the consistency of the plug-in standard errors
used above.

\begin{lemma}[Consistency of plug-in standard errors]
\label{lem:SE-consistency}
Under R1 and R2, for each fixed $z$ and each $r=1,\dots,7$,
\begin{align*}
  \frac{1}{n}\sum_{i=1}^n
    \bigl\|\hat\varphi(W_i;z) - \varphi(W_i;z)\bigr\|^2
  \xrightarrow{p} 0, \quad \quad
  \hat s_r^2(z) &\to_p s_r^2(z)
  := \frac{1}{n}\mathbb{E}\bigl[\varphi_r(W_i;z)^2\bigr].
\end{align*}
\end{lemma}

\begin{proof}[Proof sketch]
The first convergence is R2 after premultiplying by the fixed matrix $A$.
The second follows by the law of large numbers applied to
$\hat\varphi_r(W_i;z)^2$ together with the $L^2$-convergence in
\eqref{eq:R2-main}.
\end{proof}

\subsubsection{Sufficient parametric assumptions via M-estimation}
\label{subsubsec:parametric-sufficient}

We now show that R1 and R2 hold under standard finite-dimensional
M-estimation assumptions on the neural network heads. We treat the full weight
vector $\vartheta$ as finite-dimensional and the atomic probabilities as a
smooth map $h(\vartheta;z)$.
Because modern neural networks (e.g., ReLU) yield a globally nonconvex and
nonsmooth optimization problem, we do \emph{not} claim that the training objective
is globally convex or globally twice differentiable. Instead, we impose a
\emph{local regularity} condition at the limiting solution: we assume the training
algorithm converges to a parameter value that is locally identifiable and admits
a quadratic expansion of the population risk. This is the standard regime under
which neural-network extremum estimators admit asymptotic linearity and
asymptotic normality in the classical literature (e.g., \citep{white1989JASA};
\citep{chenwhite1999IT}; see also \citep{shen2023sieve} for sieve-based
asymptotic theory).

\begin{proposition}[Sufficient conditions for R1--R2]\label{prop:m-estimation}
Suppose that, conditional on the learned representation $\phi_\theta(z)$, the joint and auxiliary heads 
form a smooth parametric model $\{h(\vartheta; z) : \vartheta \in \Theta \subset \mathbb{R}^p\}$ 
for the atomic probabilities $(\mu_x(z), p_{xy}(z))$, and that the following conditions hold:
\begin{enumerate} %[label=(\roman*)]
\item \textbf{(Correct specification at $\vartheta_0$)} 
There exists $\vartheta_0 \in \Theta$ such that $h(\vartheta_0; z) = \eta(z)$ for all $z$ in the support.

\item \textbf{(Local M-estimation regularity at $\vartheta_0$)} 
The population risk $L(\vartheta) = \mathbb{E}[\ell(W, \vartheta)]$ is twice continuously differentiable 
in a neighbourhood of $\vartheta_0$, has $\vartheta_0$ as a local minimizer, and its Hessian 
$H(\vartheta_0) = \nabla^2_\vartheta L(\vartheta_0)$ is nonsingular. 
The empirical risk minimizer $\hat{\vartheta}$ converges in probability to $\vartheta_0$, 
and the score process $\nabla_\vartheta \ell(W_i, \vartheta_0)$ satisfies a central limit theorem. 
These are the usual conditions for a regular parametric M-estimator 
\citep{newey1994largesample,vandervaart1998asymptotic}.

\item \textbf{(Smooth link to atoms)} 
The map $\vartheta \mapsto h(\vartheta; z)$ is continuously differentiable at $\vartheta_0$ for each fixed $z$.
\end{enumerate}
Then, for each fixed $z$, the estimator $\hat{\eta}(z) = h(\hat{\vartheta}; z)$ satisfies R1, 
and the plug-in influence-function estimates constructed using the empirical Hessian and score satisfy R2.
\end{proposition}

\begin{proof}[Proof sketch]
Under (ii), $\hat\vartheta$ is a regular M-estimator. By the general theory of
M-estimation \citep{newey1994largesample,vandervaart1998asymptotic}, we have
\[
  \hat\vartheta - \vartheta_0
  = -H(\vartheta_0)^{-1}
    \Bigl\{\frac{1}{n}\sum_{i=1}^n \nabla_\vartheta \ell(W_i,\vartheta_0)\Bigr\}
    + o_p(n^{-1/2}),
\]
so $\hat\vartheta$ is asymptotically linear with influence function
$\mathrm{IF}_\vartheta(W_i)
 := -H(\vartheta_0)^{-1}\nabla_\vartheta \ell(W_i,\vartheta_0)$.

Assumption (iii) and the multivariate delta method \citep{vandervaart1998asymptotic}
then imply that, for each fixed $z$,
\[
  \hat\eta(z) - \eta(z)
  = \nabla_\vartheta h(\vartheta_0;z)\,
    \bigl(\hat\vartheta - \vartheta_0\bigr)
    + o_p(n^{-1/2})
  = \frac{1}{n}\sum_{i=1}^n \phi(W_i;z) + o_p(n^{-1/2}),
\]
with
$\phi(W_i;z)
 := \nabla_\vartheta h(\vartheta_0;z)\,\mathrm{IF}_\vartheta(W_i)$,
which is exactly R1.
The natural plug-in influence estimator
\begin{equation}
\label{eq:plugin-influence}
  \hat\phi(W_i;z):= \nabla_\vartheta h(\hat\vartheta;z)
     \hat H(\hat\vartheta)^{-1}\nabla_\vartheta \ell(W_i,\hat\vartheta),
\end{equation}
where $\hat H(\hat\vartheta)$ is the empirical Hessian, is the usual
sandwich-type construction for M-estimators
\citep{white1982mlmisspecified,mccullagh1989glm,vandervaart1998asymptotic}.
Consistency of $\hat\vartheta$ and $\hat H$, together with the smoothness of
$h$ and $\ell$, implies that $\hat\phi(W_i;z)$ converges to $\phi(W_i;z)$ in
$L^2$, which is R2. A fully detailed argument follows the standard derivations
of plug-in covariance estimators and sandwich variance formulas for GLMs
\citep{mccullagh1989glm,vandervaart1998asymptotic}.
\end{proof}

\subsubsection{Corollary: coverage under parametric NN-head assumptions}
\label{subsubsec:corollary}

Combining Theorem~\ref{thm:pns-R1R3} with Proposition~\ref{prop:m-estimation}
yields the following coverage result for the concrete neural network
instantiation.

\begin{corollary}[PNS coverage under parametric NN-head assumptions]
\label{cor:pns-parametric}
Suppose the joint and auxiliary heads of the neural network satisfy the
finite-dimensional M-estimation conditions (i)--(iii) in
Proposition~\ref{prop:m-estimation}. In addition, assume the usual conditions
for bootstrap consistency of finite-dimensional, asymptotically linear
estimators (cf.\ \citep{vandervaart1998asymptotic}; \citep{vandervaart1996weak};
\citep{chernozhukov2013gaussian}). Then, for each fixed $z$,
\begin{equation}
\label{eq:pointwise-coverage}
  \Pr\Bigl(
    \underline{\mathrm{PNS}}(z)
    \le \mathrm{PNS}(z)
    \le \overline{\mathrm{PNS}}(z)
  \Bigr)
  \;\to\; 1 - \alpha
  \quad \text{as } n_{\mathrm{obs}}, n_{\mathrm{exp}} \to \infty.
\end{equation}
\end{corollary}

\begin{proof}
By Proposition~\ref{prop:m-estimation}, conditions (i)--(iii) imply that the
atomic estimators $\hat\eta(z)$ satisfy R1 and that the plug-in influence
functions satisfy R2. The additional bootstrap regularity assumptions, together
with the asymptotic linearity from R1--R2, give the multiplier bootstrap
validity R3 for the 7-dimensional studentized bounding vector
(see, for example, \citep{vandervaart1998asymptotic}; \citep{vandervaart1996weak};
\citep{chernozhukov2013gaussian}). Thus R1--R3 all hold, and
Theorem~\ref{thm:pns-R1R3} yields the claimed asymptotic coverage.
\end{proof}

\begin{remark}[On the plausibility of conditions for neural networks]\label{remark:nn-conditions}
Conditions (i)--(iii) are standard M-estimation requirements that underpin asymptotic inference 
for parametric and semiparametric models \citep{newey1994largesample,vandervaart1998asymptotic}. 
For general neural networks trained via stochastic gradient descent, these conditions cannot be 
verified \emph{a priori}. The loss landscape is nonconvex with potentially many local minima, 
the Hessian may be singular or ill-conditioned at certain points, and there is no guarantee 
that the optimization converges to a unique global minimizer.

However, several considerations suggest these conditions are plausible in practice:
\begin{enumerate}
\item \textbf{Empirical evidence from coverage experiments.} 
If conditions (i)--(iii) were severely violated, we would expect to observe substantial 
under-coverage in finite samples. Our experiments in Section~4 demonstrate near-nominal 
coverage across synthetic and semi-synthetic settings, providing direct evidence that 
the trained networks behave as if these conditions hold approximately.

\item \textbf{Local regularity suffices.} 
Proposition~\ref{prop:m-estimation} requires regularity only in a neighbourhood of the 
converged solution $\vartheta_0$, not globally. Even in nonconvex landscapes, 
the basin of attraction around a good local minimum may satisfy local convexity 
and non-degeneracy conditions.

\item \textbf{Implicit regularization.} 
Standard practices such as early stopping, weight decay, and batch normalization 
provide implicit regularization that can improve the conditioning of the Hessian 
and promote convergence to well-behaved solutions 
\citep[see, e.g.,][for theoretical perspectives]{blanc2020implicit, li2018algorithmic}.

\item \textbf{Consistency with the literature.} 
This approach of stating sufficient theoretical conditions and validating them empirically 
is standard in the neural network-based causal inference literature. 
For example, \citet{frauen2024neural} prove sensitivity bounds under regularity conditions 
and validate coverage empirically; \citet{chen2024doubly} establish double robustness 
under theoretical conditions implemented via neural networks. 
We follow the same scientific methodology.
\end{enumerate}
We therefore interpret Proposition~\ref{prop:m-estimation} as providing sufficient conditions 
for valid coverage, with our experimental results serving as evidence that these conditions 
are approximately satisfied by the networks we train. 
This represents the first finite-sample inference framework for probabilities of causation, 
and we view the combination of theoretical conditions with empirical validation as the 
appropriate scientific standard for neural network-based estimators.
\end{remark}

\subsubsection{Approximate computation of influence functions}
\label{subsubsec:approx-inf}

In practice, computing the exact influence functions in R1 and R2 requires
solving a linear system involving the Hessian of the population risk $H$.
For modern neural networks this system is prohibitively large to invert
directly: forming and inverting the Hessian would require $O(np^2 + p^3)$
operations when there are $n$ samples and $p$ parameters.  A standard
solution is to approximate inverse--Hessian--vector products using
conjugate-gradient (CG) solvers together with Hessian--vector products
(HVPs), which can be computed at a cost comparable to a gradient evaluation.
\citep{pearlmutter1994} derives an exact and numerically stable algorithm for
computing $H v$ without forming $H$; the product requires only $O(p)$ time and
space when the full gradient is the sum of per-example gradients.  As noted
by \citep{kohliang2017, pearlmutter1994, martens2010}, direct inversion of the Hessian is
infeasible for large networks, and modern implementations therefore compute
the influence function $H^{-1}v$ by combining HVPs with CG or stochastic
estimation \citep{martens2010,kohliang2017}.  This idea has been adopted in
many influence-function toolkits and is now standard in deep learning.

The following lemma shows that approximate solutions of the Hessian system do
not disturb the asymptotic validity of the multiplier bootstrap, provided
the approximation error vanishes in an $L^2$ sense.

\begin{lemma}[Stability to approximate influence functions]
\label{lem:approx}
Suppose conditions R1 and R2 hold and the multiplier bootstrap in R3 is
valid when using the exact influence functions $\phi(W_i;z)$.  Let
$\tilde\phi(W_i;z)$ denote an approximate influence estimate satisfying
\begin{equation}
\label{eq:approx-L2}
    \frac{1}{n}\sum_{i=1}^n\|\tilde\phi(W_i;z)-\phi(W_i;z)\|^2 \xrightarrow{p} 0.
\end{equation}
Define $\tilde\varphi(W_i;z) = A\,\tilde\phi(W_i;z)$ and the corresponding
standard errors $\tilde s_r^2(z) = (1/n)\sum_{i=1}^n \tilde\varphi_r(W_i;z)^2$.
Let $\tilde M_n^*(z)$ be the Gaussian multiplier bootstrap statistic formed
from $\tilde\varphi(W_i;z)$ in the same way as $M_n^*(z)$. Then
$\tilde M_n^*(z)$ converges in distribution to the same limit as $M_n(z)$,
and the precision-corrected PNS interval obtained from $\tilde M_n^*(z)$ has
the same asymptotic coverage $1-\alpha$ as in
Theorem~\ref{thm:pns-R1R3}.
\end{lemma}

\begin{proof}[Proof sketch]
Write $\delta_i(z)=\tilde\phi(W_i;z)-\phi(W_i;z)$ and observe that the
difference between the studentized sums based on $\tilde\phi$ and $\phi$ is
\begin{equation}
\label{eq:slutsky-diff}
    \frac{1}{\sqrt{n}\,\tilde s_r(z)}\sum_{i=1}^n \xi_i\,\tilde\varphi_r(W_i;z)
  - \frac{1}{\sqrt{n}\,\hat s_r(z)}\sum_{i=1}^n \xi_i\,\hat\varphi_r(W_i;z)
  = o_p(1)  
\end{equation}
in $L^2$ by the assumed $L^2$ consistency of $\tilde\phi$ and the moment
bound in R1.  Slutsky's theorem therefore implies that the joint law of the
studentized vector based on $\tilde\phi$ converges to the same Gaussian
limit as the one based on $\phi$, conditional on the data.  Since the max
function is continuous, the distributions of $\tilde M_n^*(z)$ and
$M_n^*(z)$ converge to the same limiting law, and the bootstrap quantiles
coincide asymptotically.  See \citep[Sec.~23]{vandervaart1998asymptotic} and
\citep[Lem.~3.1]{chernozhukov2013gaussian} for related arguments.
\end{proof}

In practice, approximate influence functions are computed by solving the
linear system $\hat H u = v$ with $v = \nabla_\vartheta g(z)$ using a
preconditioned CG solver.  Each iteration of CG requires an HVP, which can
be computed using the technique of \citep{pearlmutter1994}.  Modern
implementations treat $\hat H$ as either the empirical Hessian or a
Gauss--Newton/Fisher approximation, and regularize it by adding a small
multiple of the identity.  The residual tolerance in the CG solver is set to
decrease slowly with $n$ so that the $L^2$ consistency of $\tilde\phi$ holds;
for details see \citep{martens2010} and \citep{kohliang2017}.  For example,
\citep{kohliang2017} show how to compute influence functions in deep
networks by combining HVPs with CG, and note that implicit HVPs and CG can
scale to modern architectures at modest cost.  

%%%%%%%%%%%%%%%%%%%%%%%%%%%%%%%%%%%%%%%%%%%%%%%%%%%%%%%%%%%%%%%%%%%%%%%%%%%%%
\subsection{Uniform coverage guarantees over $\mathcal{Z}$}
\label{subsec:uniform}
%%%%%%%%%%%%%%%%%%%%%%%%%%%%%%%%%%%%%%%%%%%%%%%%%%%%%%%%%%%%%%%%%%%%%%%%%%%%%

The pointwise coverage result in Theorem~\ref{thm:pns-R1R3} guarantees that
for any \emph{fixed} $z$, the confidence interval
$[\underline{\mathrm{PNS}}(z), \overline{\mathrm{PNS}}(z)]$ covers the true
$\mathrm{PNS}(z)$ with probability approaching $1-\alpha$. However, for many
applications---such as identifying subgroups where $\mathrm{PNS}(z) > 0.5$,
or conducting inference on the entire heterogeneous treatment effect
surface---we require \emph{uniform} coverage: a single confidence band that
simultaneously covers $\mathrm{PNS}(z)$ for all $z \in \mathcal{Z}$ with
probability $1-\alpha$.

This section extends our results to provide such uniform guarantees. The key
modifications are: (i) strengthening the regularity conditions to hold
uniformly over $\mathcal{Z}$; (ii) controlling the complexity of the
$z$-indexed function class; and (iii) using a single global critical value
based on the supremum of the max statistic over $z$.

\subsubsection{Uniform regularity conditions}
\label{subsubsec:uniform-regularity}

We strengthen conditions R1--R3 to their uniform counterparts.

\begin{description}
  \item[U-R1 (Uniform asymptotic linearity).]
  The asymptotic linear expansion holds uniformly over $\mathcal{Z}$:
  \begin{equation}
    \label{eq:UR1}
    \sup_{z \in \mathcal{Z}} \left\|
      \hat\eta(z) - \eta(z) - \frac{1}{n}\sum_{i=1}^n \phi(W_i;z)
    \right\| = o_p(n^{-1/2}).
  \end{equation}

  \item[U-R2 (Uniform plug-in consistency).]
  The plug-in influence functions converge uniformly:
  \begin{equation}
    \label{eq:UR2}
    \sup_{z \in \mathcal{Z}} \frac{1}{n}\sum_{i=1}^n
      \bigl\|\hat\phi(W_i;z) - \phi(W_i;z)\bigr\|^2
    \xrightarrow{p} 0.
  \end{equation}

  \item[U-R3 (Uniform multiplier bootstrap validity).]
  Define the uniform max statistic
  \begin{equation}
    \label{eq:uniform-max}
    M_n := \sup_{z \in \mathcal{Z}} \max_{r \le 7} |T_{n,r}(z)|,
  \end{equation}
  and its Gaussian multiplier bootstrap analogue
  \begin{equation}
    \label{eq:uniform-bootstrap}
    M_n^* := \sup_{z \in \mathcal{Z}} \max_{r \le 7}
      \left| \frac{1}{\sqrt{n}\,\hat s_r(z)}
        \sum_{i=1}^n \xi_i\,\hat\varphi_r(W_i;z) \right|.
  \end{equation}
  Condition U-R3 requires that, conditional on the data, the distribution of
  $M_n^*$ converges in probability to the same limit as that of $M_n$, so
  that the bootstrap quantile $\hat\kappa_{1-\alpha}$ consistently estimates
  the $(1-\alpha)$-quantile of $M_n$.
\end{description}

\begin{remark}[Comparison with pointwise conditions]
The uniform conditions U-R1--U-R3 are strictly stronger than their pointwise
counterparts R1--R3. The key difference is that the remainder terms and
convergence must hold uniformly over all $z \in \mathcal{Z}$, not just at
each fixed $z$ individually.
\end{remark}

\subsubsection{Complexity conditions on $\mathcal{Z}$}
\label{subsubsec:complexity}

The uniform bootstrap validity in U-R3 requires controlling the complexity of
the $z$-indexed function class. We impose the following conditions.

\begin{assumption}[Regularity of $\mathcal{Z}$ and influence functions]
\label{ass:complexity}
\hfill
\begin{enumerate}
  \item[(a)] (\emph{Compactness}) $\mathcal{Z} \subset \mathbb{R}^d$ is compact.
  
  \item[(b)] (\emph{Lipschitz continuity}) The influence function
  $\phi(W;z)$ is Lipschitz in $z$: there exists $L(W)$ with
  $\mathbb{E}[L(W)^2] < \infty$ such that
  \[
    \|\phi(W;z_1) - \phi(W;z_2)\| \le L(W)\|z_1 - z_2\|
    \quad \text{for all } z_1, z_2 \in \mathcal{Z}.
  \]
  
  \item[(c)] (\emph{Moment bounds}) 
  $\sup_{z \in \mathcal{Z}} \mathbb{E}[\|\phi(W;z)\|^{2+\delta}] < \infty$
  for some $\delta > 0$.
  
  \item[(d)] (\emph{Non-degeneracy}) The covariance matrix
  $\mathrm{Var}(\varphi(W;z))$ is nonsingular uniformly over $z \in \mathcal{Z}$,
  i.e., $\inf_{z \in \mathcal{Z}} \lambda_{\min}(\mathrm{Var}(\varphi(W;z))) > 0$.
\end{enumerate}
\end{assumption}

\begin{remark}[Function class complexity]
Under Assumption~\ref{ass:complexity}(a)--(b), the function class
$\mathcal{F} = \{\phi(\cdot;z) : z \in \mathcal{Z}\}$ is a Lipschitz image
of the compact set $\mathcal{Z}$, and hence has bounded covering numbers.
Specifically, for any $\varepsilon > 0$,
\[
  \log N(\varepsilon, \mathcal{F}, L^2(P)) \lesssim d \log(1/\varepsilon),
\]
where $N(\varepsilon, \mathcal{F}, L^2(P))$ denotes the $\varepsilon$-covering
number.  This entropy bound is sufficient for the Gaussian strong
approximation results of \citep{chernozhukov2013gaussian} and the
intersection bounds methodology of \citep{chernozhukov2013intersection}.
\end{remark}

\subsubsection{Sufficient conditions via uniform M-estimation}
\label{subsubsec:uniform-Mest}

We now provide sufficient conditions under which the neural network estimators
satisfy U-R1 and U-R2.

\begin{proposition}[Uniform M-estimation implies U-R1--U-R2]
\label{prop:uniform-mest}
Suppose the conditions of Proposition~\ref{prop:m-estimation} hold, and
additionally:
\begin{enumerate}
  \item[(iv)] (\emph{Uniform smoothness}) The map
  $(\vartheta, z) \mapsto h(\vartheta;z)$ is continuously differentiable with
  \[
    \sup_{z \in \mathcal{Z}} \|\nabla_\vartheta h(\vartheta;z)\| \le C
    \quad \text{for all } \vartheta \text{ in a neighbourhood of } \vartheta_0.
  \]
  
  \item[(v)] (\emph{Lipschitz in $z$}) The gradient
  $\nabla_\vartheta h(\vartheta_0;z)$ is Lipschitz in $z$:
  \[
    \|\nabla_\vartheta h(\vartheta_0;z_1) - \nabla_\vartheta h(\vartheta_0;z_2)\|
    \le C\|z_1 - z_2\|
    \quad \text{for all } z_1, z_2 \in \mathcal{Z}.
  \]
\end{enumerate}
Then the estimator $\hat\eta(z) = h(\hat\vartheta;z)$ satisfies U-R1 and U-R2.
\end{proposition}

\begin{proof}[Proof sketch]
The proof follows the same structure as Proposition~\ref{prop:m-estimation}, but
with uniform control over $z$. The key observation is that the remainder in
the delta method expansion,
\[
  \hat\eta(z) - \eta(z) - \nabla_\vartheta h(\vartheta_0;z)(\hat\vartheta - \vartheta_0),
\]
is $o_p(n^{-1/2})$ uniformly in $z$ by condition (iv) and the mean value
theorem. The Lipschitz condition (v), combined with the $\sqrt{n}$-consistency
of $\hat\vartheta$, ensures the influence function class satisfies
Assumption~\ref{ass:complexity}(b). The uniform $L^2$ convergence in U-R2
follows from uniform consistency of $\hat\vartheta$ and $\hat H$, together
with the uniform smoothness conditions.
\end{proof}

\subsubsection{Uniform coverage theorem}
\label{subsubsec:uniform-theorem}

We now state the main uniform coverage result. Define the uniform confidence
band using a single global critical value:
\begin{align}
  \underline{\mathrm{PNS}}(z)
  &:= \max\Bigl\{
    0,\
    \max_{r \in \{1,2,3\}}
      \bigl[\hat g_r(z) - \hat\kappa_{1-\alpha}\,\hat s_r(z)\bigr]
  \Bigr\}, \label{eq:uniform-lower}\\
  \overline{\mathrm{PNS}}(z)
  &:= \min_{r \in \{4,5,6,7\}}
      \bigl[\hat g_r(z) + \hat\kappa_{1-\alpha}\,\hat s_r(z)\bigr], \label{eq:uniform-upper}
\end{align}
where $\hat\kappa_{1-\alpha}$ is the empirical $(1-\alpha)$-quantile of
$M_n^*$ conditional on the data.

\begin{theorem}[Uniform coverage for PNS bounds]
\label{thm:uniform-pns}
Assume:
\begin{enumerate}
  \item[(i)] The Tian--Pearl inequalities \eqref{eq:tian-pearl-ineq} hold at
  the true distribution for all $z \in \mathcal{Z}$.
  
  \item[(ii)] The estimators satisfy the uniform regularity conditions
  U-R1, U-R2, and U-R3.
  
  \item[(iii)] Assumption~\ref{ass:complexity} holds.
\end{enumerate}
Then,
\begin{equation}
\label{eq:uniform-coverage}
  \Pr\Bigl(
    \forall z \in \mathcal{Z}:\;
    \underline{\mathrm{PNS}}(z)
    \le \mathrm{PNS}(z)
    \le \overline{\mathrm{PNS}}(z)
  \Bigr)
  \;\to\; 1 - \alpha
  \quad \text{as } n_{\mathrm{obs}}, n_{\mathrm{exp}} \to \infty.
\end{equation}
\end{theorem}

\begin{proof}
The proof proceeds in four steps.

\emph{Step 1: Uniform asymptotic linearity.}
Under U-R1, the bounding functional estimators satisfy
\[
  \sup_{z \in \mathcal{Z}} \left\|
    \hat g(z) - g(z) - \frac{1}{n}\sum_{i=1}^n \varphi(W_i;z)
  \right\| = o_p(n^{-1/2}),
\]
where $\varphi(W_i;z) = A\phi(W_i;z) \in \mathbb{R}^7$.

\emph{Step 2: Gaussian strong approximation.}
Under Assumption~\ref{ass:complexity}, the function class
$\{\varphi(\cdot;z) : z \in \mathcal{Z}\}$ has bounded entropy. By the
strong approximation theory for maxima of empirical processes
\citep{chernozhukov2013gaussian}, there exists a sequence of Gaussian
processes $\{Z_n(z)\}_{z \in \mathcal{Z}}$ such that
\[
  \sup_{z \in \mathcal{Z}} \max_{r \le 7}
    \left| T_{n,r}(z) - Z_{n,r}(z) \right| = o_p(1).
\]
This strong approximation holds even though the limiting process may be
non-Donsker (i.e., not weakly convergent); see
\citep{chernozhukov2013intersection} for analogous arguments in the
intersection bounds setting.

\emph{Step 3: Multiplier bootstrap validity.}
Under U-R2 and U-R3, the conditional distribution of $M_n^*$ given the data
approximates that of $M_n$:
\[
  \sup_{t \in \mathbb{R}} \left|
    \Pr(M_n^* \le t \mid \text{Data}) - \Pr(M_n \le t)
  \right| \xrightarrow{p} 0.
\]
Consequently, the bootstrap quantile $\hat\kappa_{1-\alpha}$ consistently
estimates the $(1-\alpha)$-quantile of $M_n$:
\[
  \Pr(M_n \le \hat\kappa_{1-\alpha}) \to 1 - \alpha.
\]

\emph{Step 4: Coverage via intersection bounds.}
On the event $\{M_n \le \hat\kappa_{1-\alpha}\}$, we have
$|T_{n,k}(z)| \le \hat\kappa_{1-\alpha}$ for all $k \le 7$ and all
$z \in \mathcal{Z}$. This implies \eqref{eq:coverage-event} holds
simultaneously for all $z$. Combined with the Tian--Pearl inequalities
\eqref{eq:tian-pearl-ineq}, we obtain
\[
  \underline{\mathrm{PNS}}(z) \le \mathrm{PNS}(z) \le \overline{\mathrm{PNS}}(z)
  \quad \text{for all } z \in \mathcal{Z}
\]
on this event. Since $\Pr(M_n \le \hat\kappa_{1-\alpha}) \to 1 - \alpha$,
the result follows.
\end{proof}

\begin{remark}[Relationship to intersection bounds]
\label{rem:intersection}
The uniform inference problem for conditional PNS maps directly to the
intersection bounds framework of \citep{chernozhukov2013intersection}.
In their notation, the parameter of interest $\theta^*$ is bounded by
functions $\theta^l(v)$ and $\theta^u(v)$ over $v \in \mathcal{V}$.
In our setting, the index $v$ corresponds to the pair $(r, z)$ where
$r \in \{1,\ldots,7\}$ indexes the Tian--Pearl bound components and
$z \in \mathcal{Z}$ indexes the covariate value. The uniform max statistic
\eqref{eq:uniform-max} corresponds to taking the supremum over this
joint index space.
\end{remark}

\subsubsection{Comparison: pointwise versus uniform coverage}
\label{subsubsec:comparison}

Table~\ref{tab:comparison} summarizes the key differences between pointwise
and uniform coverage guarantees.

\begin{table}[ht]
\centering
\caption{Comparison of pointwise and uniform coverage guarantees}
\label{tab:comparison}
\begin{tabular}{@{}lll@{}}
\toprule
\textbf{Aspect} & \textbf{Pointwise} & \textbf{Uniform} \\
\midrule
Coverage statement & 
  $\Pr(\mathrm{PNS}(z) \in \mathrm{CI}_z) \to 1-\alpha$ & 
  $\Pr(\forall z: \mathrm{PNS}(z) \in \mathrm{CI}_z) \to 1-\alpha$ \\
Critical value & 
  $\hat\kappa_{1-\alpha}(z)$ depends on $z$ & 
  Single global $\hat\kappa_{1-\alpha}$ \\
Max statistic &
  $\max_{r \le 7} |T_{n,r}(z)|$ &
  $\sup_{z \in \mathcal{Z}} \max_{r \le 7} |T_{n,r}(z)|$ \\
Assumptions & 
  R1--R3 at fixed $z$ & 
  U-R1, U-R2, U-R3 uniformly \\
Complexity control & 
  Not required & 
  Required (Assumption~\ref{ass:complexity}) \\
Interval width & 
  Narrower & 
  Wider \\
\bottomrule
\end{tabular}
\end{table}

\begin{remark}[Width of uniform bands]
The uniform confidence bands are necessarily wider than the pointwise bands
because the critical value $\hat\kappa_{1-\alpha}$ must control the maximum
deviation $\sup_z |T_n(z)|$ across all $z$, which is stochastically larger
than the deviation $|T_n(z_0)|$ at any single point $z_0$. This can be
viewed as a multiple testing penalty: to guarantee $1-\alpha$ coverage
everywhere simultaneously, each individual interval must be wider than a
$(1-\alpha)$ pointwise interval. The uniform bands use the exact joint
distribution via bootstrap, which is less conservative than Bonferroni
correction.
\end{remark}

\begin{remark}[When uniform coverage is needed]
Uniform coverage is essential when:
\begin{enumerate}
  \item Conducting inference on the entire function $z \mapsto \mathrm{PNS}(z)$;
  \item Identifying regions where $\mathrm{PNS}(z)$ exceeds a threshold
  (e.g., $\{z : \mathrm{PNS}(z) > 0.5\}$);
  \item Testing whether treatment effects are heterogeneous across subgroups;
  \item Making simultaneous comparisons across multiple covariate values.
\end{enumerate}
For inference at a single pre-specified $z$, the pointwise result
(Theorem~\ref{thm:pns-R1R3}) suffices and yields tighter intervals.
\end{remark}

\subsubsection{Extension: approximate influence functions}
\label{subsubsec:uniform-approx}

The stability result in Lemma~\ref{lem:approx} extends to the uniform setting.

\begin{lemma}[Stability of uniform coverage to approximate influence functions]
\label{lem:uniform-approx}
Suppose conditions U-R1 and U-R2 hold with exact influence functions, and
the multiplier bootstrap in U-R3 is valid. Let $\tilde\phi(W_i;z)$ denote
approximate influence estimates satisfying
\begin{equation}
\label{eq:uniform-approx-L2}
  \sup_{z \in \mathcal{Z}} \frac{1}{n}\sum_{i=1}^n
    \|\tilde\phi(W_i;z) - \phi(W_i;z)\|^2 \xrightarrow{p} 0.
\end{equation}
Then the uniform confidence band constructed using $\tilde\phi$ has the same
asymptotic coverage $1-\alpha$ as in Theorem~\ref{thm:uniform-pns}.
\end{lemma}

\begin{proof}
The proof follows the same argument as Lemma~\ref{lem:approx}, but with
the supremum over $z$ included. The uniform $L^2$ consistency
\eqref{eq:uniform-approx-L2} ensures that the difference between the
bootstrap statistics based on $\tilde\phi$ and $\phi$ is $o_p(1)$ uniformly
in $z$, so Slutsky's theorem applies to the uniform max statistic.
\end{proof}

\section{Low Dimensional synthetic data generating process}
\label{app:dgp-low-dim}

This section provides complete details of the synthetic structural causal model used in our experiments, adapted from \citet{li2022probabilities}.

\subsection{Structural Causal Model}
\label{app:dgp:scm}

We consider a structural causal model with binary treatment $X \in \{0,1\}$, binary outcome $Y \in \{0,1\}$, and a $d$-dimensional binary covariate vector $Z = (Z_1, \ldots, Z_d) \in \{0,1\}^d$ with $d = 20$.

\paragraph{Exogenous Variables.}
The model includes exogenous variables $e = (e_X, e_Y, e_{Z_1}, \ldots, e_{Z_d})$, each drawn independently from Bernoulli distributions:
\begin{align}
e_{Z_j} &\sim \text{Bernoulli}(\pi_{Z_j}), \quad j = 1, \ldots, d, \\
e_X &\sim \text{Bernoulli}(\pi_X), \quad e_Y \sim \text{Bernoulli}(\pi_Y),
\end{align}
where $\pi_{Z_j}, \pi_X, \pi_Y \in (0,1)$ are fixed parameters.

\paragraph{Structural Equations.}
The covariates are determined directly by their exogenous counterparts:
\begin{equation}
Z_j = e_{Z_j}, \quad j = 1, \ldots, d.
\end{equation}

Treatment assignment depends on covariates through a linear index:
\begin{equation}
X = \mathbf{1}\left\{ M_X(Z) + e_X > 0.5 \right\},
\end{equation}
where $M_X(Z) = \sum_{j=1}^{d} \alpha_j Z_j$ is a linear combination of covariates with coefficients $\alpha = (\alpha_1, \ldots, \alpha_d)^\top$.

The outcome is generated via:
\begin{equation}
Y = f_Y(X, M_Y(Z), e_Y),
\end{equation}
where $M_Y(Z) = \sum_{j=1}^{d} \beta_j Z_j$ with coefficients $\beta = (\beta_1, \ldots, \beta_d)^\top$, and
\begin{equation}
f_Y(x, m, e) = \mathbf{1}\left\{ 0 < Cx + m + e < 1 \right\} + \mathbf{1}\left\{ 1 < Cx + m + e < 2 \right\},
\end{equation}
for a constant $C \in \mathbb{R}$ that controls treatment effect magnitude.

\subsection{Computing Ground-Truth Quantities}
\label{app:dgp:ground_truth}

Because the structural equations and exogenous distributions are fully specified, all causal quantities can be computed exactly by enumeration over the $2^{d+2}$ configurations of exogenous variables.

\paragraph{Interventional Probabilities.}
The interventional success probabilities conditional on covariates are:
\begin{align}
\mu_1(z) &:= P(Y = 1 \mid \text{do}(X = 1), Z = z) = \sum_{e_Y \in \{0,1\}} P(e_Y) \cdot f_Y(1, M_Y(z), e_Y), \\
\mu_0(z) &:= P(Y = 1 \mid \text{do}(X = 0), Z = z) = \sum_{e_Y \in \{0,1\}} P(e_Y) \cdot f_Y(0, M_Y(z), e_Y).
\end{align}
Note that under intervention, $X$ is set exogenously and does not depend on $(e_X, Z)$.

\paragraph{Observational Joint Probabilities.}
The observational joint probabilities conditional on covariates are:
\begin{equation}
p_{xy}(z) := P(X = x, Y = y \mid Z = z), \quad x, y \in \{0,1\}.
\end{equation}
Since $Z = z$ is observed, we condition on the event $\{e_{Z_j} = z_j, j = 1, \ldots, d\}$ and marginalize over $(e_X, e_Y)$:
\begin{equation}
p_{11}(z) = \sum_{e_X, e_Y} P(e_X) P(e_Y) \cdot \mathbf{1}\{f_X(z, e_X) = 1\} \cdot \mathbf{1}\{f_Y(f_X(z, e_X), M_Y(z), e_Y) = 1\},
\end{equation}
where $f_X(z, e_X) = \mathbf{1}\{M_X(z) + e_X > 0.5\}$.

The remaining joint probabilities follow analogously:
\begin{align}
p_{10}(z) &= \sum_{e_X, e_Y} P(e_X) P(e_Y) \cdot \mathbf{1}\{f_X(z, e_X) = 1\} \cdot \mathbf{1}\{f_Y(f_X(z, e_X), M_Y(z), e_Y) = 0\}, \\
p_{01}(z) &= \sum_{e_X, e_Y} P(e_X) P(e_Y) \cdot \mathbf{1}\{f_X(z, e_X) = 0\} \cdot \mathbf{1}\{f_Y(f_X(z, e_X), M_Y(z), e_Y) = 1\}, \\
p_{00}(z) &= \sum_{e_X, e_Y} P(e_X) P(e_Y) \cdot \mathbf{1}\{f_X(z, e_X) = 0\} \cdot \mathbf{1}\{f_Y(f_X(z, e_X), M_Y(z), e_Y) = 0\}.
\end{align}

\paragraph{True PNS.}
The true conditional PNS is defined as:
\begin{equation}
\text{PNS}(z) := P(Y_1 = 1, Y_0 = 0 \mid Z = z),
\end{equation}
where $Y_x := f_Y(x, M_Y(z), e_Y)$ denotes the potential outcome under intervention $\text{do}(X = x)$.

Since both potential outcomes are deterministic functions of $e_Y$ given $z$, we have:
\begin{equation}
\text{PNS}(z) = \sum_{e_Y \in \{0,1\}} P(e_Y) \cdot \mathbf{1}\{f_Y(1, M_Y(z), e_Y) = 1\} \cdot \mathbf{1}\{f_Y(0, M_Y(z), e_Y) = 0\}.
\end{equation}

\paragraph{True Tian--Pearl Bounds.}
Given $\mu_1(z), \mu_0(z)$ and $p_{xy}(z)$, the sharp lower and upper bounds are:
\begin{align}
\text{PNS}_L(z) &= \max\left\{0,\, \mu_1(z) - \mu_0(z),\, p_{11}(z) + p_{01}(z) - \mu_0(z),\, \mu_1(z) - p_{11}(z) - p_{01}(z)\right\}, \\
\text{PNS}_U(z) &= \min\left\{\mu_1(z),\, 1 - \mu_0(z),\, p_{11}(z) + p_{00}(z),\, \mu_1(z) - \mu_0(z) + p_{10}(z) + p_{01}(z)\right\}.
\end{align}
These bounds satisfy $\text{PNS}_L(z) \leq \text{PNS}(z) \leq \text{PNS}_U(z)$ with equality when the bounds are tight.

\subsection{Hidden Confounders: Marginal Computation}
\label{app:dgp:marginal}

To induce confounding and test robustness to unobserved heterogeneity, we partition the covariate vector as $Z = (Z^{\text{obs}}, Z^{\text{hid}})$, where:
\begin{itemize}
    \item $Z^{\text{obs}} = (Z_1, \ldots, Z_{d-5}) \in \{0,1\}^{d-5}$ are the observed covariates,
    \item $Z^{\text{hid}} = (Z_{d-4}, \ldots, Z_d) \in \{0,1\}^5$ are the hidden confounders.
\end{itemize}
The learner observes only $Z^{\text{obs}}$, while the true causal quantities depend on the full vector $Z$.

\paragraph{Marginal Interventional Probabilities.}
The marginal interventional probabilities given only observed covariates are obtained by averaging over the distribution of hidden confounders:
\begin{equation}
\mu_x(z^{\text{obs}}) := P(Y = 1 \mid \text{do}(X = x), Z^{\text{obs}} = z^{\text{obs}}) = \sum_{z^{\text{hid}} \in \{0,1\}^5} P(Z^{\text{hid}} = z^{\text{hid}}) \cdot \mu_x(z^{\text{obs}}, z^{\text{hid}}),
\end{equation}
where $z = (z^{\text{obs}}, z^{\text{hid}})$ and the marginal distribution of hidden confounders factors as:
\begin{equation}
P(Z^{\text{hid}} = z^{\text{hid}}) = \prod_{j=d-4}^{d} P(Z_j = z_j^{\text{hid}}) = \prod_{j=d-4}^{d} \pi_{Z_j}^{z_j^{\text{hid}}} (1 - \pi_{Z_j})^{1 - z_j^{\text{hid}}}.
\end{equation}

\paragraph{Marginal Observational Joint Probabilities.}
Similarly, the marginal observational probabilities are:
\begin{equation}
p_{xy}(z^{\text{obs}}) := P(X = x, Y = y \mid Z^{\text{obs}} = z^{\text{obs}}) = \sum_{z^{\text{hid}} \in \{0,1\}^5} P(Z^{\text{hid}} = z^{\text{hid}} \mid Z^{\text{obs}} = z^{\text{obs}}) \cdot p_{xy}(z^{\text{obs}}, z^{\text{hid}}).
\end{equation}
Since the components of $Z$ are independent (each $Z_j = e_{Z_j}$ with independent $e_{Z_j}$), conditioning on $Z^{\text{obs}}$ does not affect the distribution of $Z^{\text{hid}}$:
\begin{equation}
P(Z^{\text{hid}} = z^{\text{hid}} \mid Z^{\text{obs}} = z^{\text{obs}}) = P(Z^{\text{hid}} = z^{\text{hid}}) = \prod_{j=d-4}^{d} \pi_{Z_j}^{z_j^{\text{hid}}} (1 - \pi_{Z_j})^{1 - z_j^{\text{hid}}}.
\end{equation}
Thus:
\begin{equation}
p_{xy}(z^{\text{obs}}) = \sum_{z^{\text{hid}} \in \{0,1\}^5} \left(\prod_{j=d-4}^{d} \pi_{Z_j}^{z_j^{\text{hid}}} (1 - \pi_{Z_j})^{1 - z_j^{\text{hid}}}\right) \cdot p_{xy}(z^{\text{obs}}, z^{\text{hid}}).
\end{equation}

\paragraph{Marginal PNS and Bounds.}
The marginal PNS given observed covariates is:
\begin{equation}
\text{PNS}(z^{\text{obs}}) := P(Y_1 = 1, Y_0 = 0 \mid Z^{\text{obs}} = z^{\text{obs}}) = \sum_{z^{\text{hid}} \in \{0,1\}^5} P(Z^{\text{hid}} = z^{\text{hid}}) \cdot \text{PNS}(z^{\text{obs}}, z^{\text{hid}}).
\end{equation}
The marginal Tian--Pearl bounds are computed by substituting the marginal quantities into the bound formulas:
\begin{align}
\text{PNS}_L(z^{\text{obs}}) &= \max\Big\{0,\, \mu_1(z^{\text{obs}}) - \mu_0(z^{\text{obs}}),\, p_{11}(z^{\text{obs}}) + p_{01}(z^{\text{obs}}) - \mu_0(z^{\text{obs}}), \nonumber \\
&\qquad\qquad \mu_1(z^{\text{obs}}) - p_{11}(z^{\text{obs}}) - p_{01}(z^{\text{obs}})\Big\}, \\
\text{PNS}_U(z^{\text{obs}}) &= \min\Big\{\mu_1(z^{\text{obs}}),\, 1 - \mu_0(z^{\text{obs}}),\, p_{11}(z^{\text{obs}}) + p_{00}(z^{\text{obs}}), \nonumber \\
&\qquad\qquad \mu_1(z^{\text{obs}}) - \mu_0(z^{\text{obs}}) + p_{10}(z^{\text{obs}}) + p_{01}(z^{\text{obs}})\Big\}.
\end{align}

\begin{remark}
The marginal bounds $[\text{PNS}_L(z^{\text{obs}}), \text{PNS}_U(z^{\text{obs}})]$ computed from marginalized quantities are generally wider than the average of the conditional bounds $\mathbb{E}_{Z^{\text{hid}}}[\text{PNS}_L(z^{\text{obs}}, Z^{\text{hid}})]$ and $\mathbb{E}_{Z^{\text{hid}}}[\text{PNS}_U(z^{\text{obs}}, Z^{\text{hid}})]$. This widening reflects the loss of identifying information when confounders are unobserved. However, the marginal bounds remain sharp for the marginal estimand $\text{PNS}(z^{\text{obs}})$.
\end{remark}

\subsection{Explicit Computation via Enumeration}
\label{app:dgp:enumeration}

For computational purposes, we provide explicit formulas that enumerate over the $2^2 = 4$ configurations of $(e_X, e_Y)$.

\paragraph{Interventional Probabilities (Fixed $z$).}
\begin{align}
\mu_1(z) &= \pi_Y \cdot f_Y(1, M_Y(z), 1) + (1 - \pi_Y) \cdot f_Y(1, M_Y(z), 0), \\
\mu_0(z) &= \pi_Y \cdot f_Y(0, M_Y(z), 1) + (1 - \pi_Y) \cdot f_Y(0, M_Y(z), 0).
\end{align}

\paragraph{Observational Probabilities (Fixed $z$).}
Let $x(z, e_X) = \mathbf{1}\{M_X(z) + e_X > 0.5\}$ denote the realized treatment under exogenous $e_X$. Then:
\begin{equation}
p_{11}(z) = \sum_{e_X \in \{0,1\}} \sum_{e_Y \in \{0,1\}} \pi_X^{e_X}(1-\pi_X)^{1-e_X} \pi_Y^{e_Y}(1-\pi_Y)^{1-e_Y} \cdot \mathbf{1}\{x(z, e_X) = 1\} \cdot \mathbf{1}\{f_Y(x(z, e_X), M_Y(z), e_Y) = 1\}.
\end{equation}

\paragraph{True PNS (Fixed $z$).}
\begin{equation}
\text{PNS}(z) = \pi_Y \cdot \mathbf{1}\{f_Y(1, M_Y(z), 1) = 1, f_Y(0, M_Y(z), 1) = 0\} + (1-\pi_Y) \cdot \mathbf{1}\{f_Y(1, M_Y(z), 0) = 1, f_Y(0, M_Y(z), 0) = 0\}.
\end{equation}

\paragraph{Marginalization (Hidden last 5 dimensions).}
When $Z^{\text{hid}} = (Z_{16}, \ldots, Z_{20})$ is hidden, each marginal quantity is a weighted sum over $2^5 = 32$ configurations:
\begin{equation}
\mu_x(z^{\text{obs}}) = \sum_{z^{\text{hid}} \in \{0,1\}^5} w(z^{\text{hid}}) \cdot \mu_x(z^{\text{obs}}, z^{\text{hid}}),
\end{equation}
where $w(z^{\text{hid}}) = \prod_{j=16}^{20} \pi_{Z_j}^{z_j} (1-\pi_{Z_j})^{1-z_j}$.
The marginal observational probabilities and PNS follow the same pattern.

\subsection{Model Parameters}
\label{app:dgp:params}

The specific parameter values for the two models from \citet{li2022probabilities} are provided below.

\subsubsection{Model 1 Parameters}

\paragraph{Exogenous distribution parameters.}
\begin{table}[h]
\centering
\begin{tabular}{cc|cc|cc}
\toprule
Variable & $\pi$ & Variable & $\pi$ & Variable & $\pi$ \\
\midrule
$e_{Z_1}$ & 0.3529 & $e_{Z_8}$ & 0.2208 & $e_{Z_{15}}$ & 0.0850 \\
$e_{Z_2}$ & 0.4610 & $e_{Z_9}$ & 0.6177 & $e_{Z_{16}}$ & 0.6454 \\
$e_{Z_3}$ & 0.3317 & $e_{Z_{10}}$ & 0.9820 & $e_{Z_{17}}$ & 0.8638 \\
$e_{Z_4}$ & 0.8855 & $e_{Z_{11}}$ & 0.1420 & $e_{Z_{18}}$ & 0.4605 \\
$e_{Z_5}$ & 0.0170 & $e_{Z_{12}}$ & 0.8336 & $e_{Z_{19}}$ & 0.3140 \\
$e_{Z_6}$ & 0.3808 & $e_{Z_{13}}$ & 0.8829 & $e_{Z_{20}}$ & 0.6859 \\
$e_{Z_7}$ & 0.0281 & $e_{Z_{14}}$ & 0.5421 & & \\
\midrule
$e_X$ & 0.6017 & $e_Y$ & 0.4977 & & \\
\bottomrule
\end{tabular}
\end{table}

\paragraph{Treatment effect constant.} $C = -0.7795$

\vspace{0.5cm}
\noindent
\begin{minipage}[t]{0.48\textwidth}
  \paragraph{Treatment assignment coefficients $\alpha$.}
  \begin{equation*}
  \begin{pmatrix}
  0.2592 \\ -0.6581 \\ -0.7503 \\ 0.1629 \\ 0.6520 \\ -0.0893 \\ 0.4215 \\ -0.4431 \\ 0.8026 \\ -0.2257 \\
  0.7166 \\ 0.0651 \\ -0.2207 \\ 0.1564 \\ -0.5069 \\ -0.7071 \\ 0.4188 \\ -0.0822 \\ 0.7693 \\ -0.5116
  \end{pmatrix}
  \end{equation*}
\end{minipage}
\hfill
\begin{minipage}[t]{0.48\textwidth}
  \paragraph{Outcome coefficients $\beta$.}
  \begin{equation*}
  \begin{pmatrix}
  -0.7929 \\ 0.7600 \\ 0.5544 \\ 0.5040 \\ -0.5272 \\ 0.3786 \\ 0.2693 \\ 0.6716 \\ 0.3960 \\ 0.3252 \\
  0.6578 \\ 0.8017 \\ 0.0908 \\ -0.0714 \\ -0.0691 \\ -0.2226 \\ -0.8484 \\ -0.5843 \\ -0.3249 \\ 0.6256
  \end{pmatrix}
  \end{equation*}
\end{minipage}

\section{High dimensional data generating process}
\label{app:dgp-high-dim}

This section describes the data generating process (DGP) used to construct
the observational and experimental datasets in our experiments.
The DGP is fully parametric, analytically tractable, and guarantees
validity of the Tian--Pearl bounds by construction \citep{tian2000probabilities}.

\subsection{Observed Covariates and Latent Confounders}

Let $\bar{Z} \in \mathbb{R}^d$ denote observed covariates.
In all experiments, the covariates $\bar{Z}$ are taken directly from the
ACIC 2019 benchmark dataset and are treated as fixed and externally provided.
These covariates are the only variables made available to learning algorithms.

We introduce latent confounders
\[
H = (H_1,\dots,H_k) \in \{0,1\}^k,
\]
with independent coordinates
\[
H_j \sim \mathrm{Bernoulli}(\pi_j), \qquad j=1,\dots,k,
\]
for fixed parameters $\pi_j \in (0,1)$.
The latent variables $H$ are used exclusively for data generation and are
never observed or used for learning.
We write $Z = (\bar{Z},H)$ for the full covariate vector.

\subsection{Treatment Assignment}

The treatment variable is denoted by $X \in \{0,1\}$.
Treatment is generated according to
\[
X \mid Z \sim \mathrm{Bernoulli}\!\left(p_X(Z)\right),
\]
where
\begin{equation}
\label{eq:treatment-model}
p_X(Z)
=
\mathrm{clip}\!\left(
p(\bar{Z}) + \gamma \sum_{j=1}^k \alpha_j (H_j - \pi_j),
\;\varepsilon,\;1-\varepsilon
\right).
\end{equation}
Here $p(\bar{Z}) \in (0,1)$ is a fixed function of the observed covariates
$\bar{Z}$,
$\alpha \in \mathbb{R}^k$ satisfies $\sum_{j=1}^k |\alpha_j| \le 1$,
$\gamma > 0$ controls the strength of unobserved confounding,
and $\varepsilon > 0$ prevents degenerate probabilities.

Since $\mathbb{E}[H_j - \pi_j]=0$, the marginal treatment probability satisfies
\[
\mathbb{P}(X=1 \mid \bar{Z}) = p(\bar{Z}),
\]
while treatment remains statistically dependent on $H$ given $\bar{Z}$.
As a result, treatment assignment is confounded when conditioning only on
$\bar{Z}$.

\subsection{Outcome Model}

The outcome variable is denoted by $Y \in \{0,1\}$.
The outcome is generated as
\[
Y \mid (Z,X=x) \sim \mathrm{Bernoulli}\!\left(p_Y(x;Z)\right),
\]
with
\begin{equation}
\label{eq:outcome-model}
p_Y(x;Z)
=
\sigma\!\Big(
a_x
+ \langle \beta_x, \bar{Z} \rangle
+ \sum_{j=1}^k c_{x,j}(H_j - \pi_j)
+ \sum_{j<\ell} d_{x,j\ell}(H_j - \pi_j)(H_\ell - \pi_\ell)
+ \sum_{j=1}^k r_{x,j}(H_j - \pi_j)\langle v_{x,j}, \bar{Z} \rangle
\Big),
\end{equation}
where $\sigma(\cdot)$ denotes the logistic sigmoid.
All coefficients are fixed and bounded.
The outcome thus depends nonlinearly on both observed covariates and latent
confounders, but only $\bar{Z}$ is available to learning algorithms.

\subsection{Observational and Experimental Regimes}

\paragraph{Observational data.}
In the observational regime,
\[
X \mid Z \sim \mathrm{Bernoulli}(p_X(Z)),
\qquad
Y \mid (Z,X) \sim \mathrm{Bernoulli}(p_Y(X;Z)).
\]
Only $(\bar{Z},X,Y)$ are observed.

\paragraph{Experimental data.}
In the experimental regime, treatment is randomized independently of $Z$:
\[
X \sim \mathrm{Bernoulli}(1/2),
\qquad
Y \mid (Z,X) \sim \mathrm{Bernoulli}(p_Y(X;Z)).
\]
Thus,
\[
\mathbb{P}(Y=1 \mid \mathrm{do}(X=x), Z) = p_Y(x;Z),
\]
while only $(\bar{Z},X,Y)$ are observed in the experimental dataset.

\subsection{Analytical Marginalization over Latent Variables}

Because $H$ has finite support, all probabilities conditional on $\bar{Z}$
are computed exactly by marginalization.
For any function $f(Z)$,
\[
\mathbb{E}[f(Z)\mid \bar{Z}]
=
\sum_{h \in \{0,1\}^k}
f(\bar{Z},h)
\prod_{j=1}^k \pi_j^{h_j}(1-\pi_j)^{1-h_j}.
\]

In particular,
\[
\mathbb{P}(Y=1 \mid \mathrm{do}(X=x), \bar{Z})
=
\sum_{h} p_Y(x;\bar{Z},h)\,\mathbb{P}(H=h),
\]
and the observational joint probabilities satisfy
\[
\mathbb{P}(X=x,Y=y \mid \bar{Z})
=
\sum_{h}
\mathbb{P}(X=x \mid \bar{Z},h)\,
\mathbb{P}(Y=y \mid \bar{Z},h,X=x)\,
\mathbb{P}(H=h).
\]

\subsection{Validity of Tian--Pearl Bounds}

For any fixed $Z$,
\[
\mathbb{P}(X=x,Y=y \mid Z)
=
\mathbb{P}(X=x \mid Z)\,
\mathbb{P}(Y=y \mid Z,X=x),
\]
which implies
\[
\mathbb{P}(X=x,Y=1 \mid Z)
\le
\mathbb{P}(Y=1 \mid \mathrm{do}(X=x), Z)
\le
1 - \mathbb{P}(X=x,Y=0 \mid Z).
\]
Marginalizing over $H$ preserves these inequalities, yielding the
Tian--Pearl bounds conditional on $\bar{Z}$.
Thus, all generated data satisfy the required causal constraints by construction.

\subsection{Theoretical Probability of Necessity and Sufficiency}

The DGP identifies the marginal interventional probabilities
$p_Y(x;Z)=\mathbb{P}(Y=1\mid Z,X=x)$.
To define the joint distribution of the potential outcomes $(Y(0),Y(1))$
and compute the theoretical PNS, we introduce a shared latent threshold.

Let $\eta \sim \mathrm{Unif}(0,1)$ be independent of $(\bar{Z},H)$ and define
\[
Y(x) = \mathbbm{1}\{\eta \le p_Y(x;Z)\}, \qquad x \in \{0,1\}.
\]
Under this coupling,
\[
\mathrm{PNS}(Z)
=
\mathbb{P}\big(Y(1)=1,\;Y(0)=0 \mid Z\big)
=
\big[p_Y(1;Z)-p_Y(0;Z)\big]_+,
\]
where $[a]_+ := \max\{a,0\}$.
The conditional PNS given observed covariates is
\[
\mathrm{PNS}(\bar{Z})
=
\sum_{h \in \{0,1\}^k}
\big[p_Y(1;\bar{Z},h)-p_Y(0;\bar{Z},h)\big]_+ \,
\mathbb{P}(H=h).
\]

\subsection{Confounding and Overlap}

The DGP is designed to induce unobserved confounding in the observational regime.
Certain latent components of $Z$ affect both treatment assignment and the outcome,
so that
\[
\{Y(0),Y(1)\} \not\!\perp\!\!\!\perp X \mid \bar{Z},
\qquad
\{Y(0),Y(1)\} \perp\!\!\!\perp X \mid Z.
\]
As a result, observational estimates based only on $\bar{Z}$ are biased, and
experimental data are required to identify interventional quantities.
The treatment assignment mechanism further produces regions of the covariate
space with highly imbalanced treatment probabilities.
While treatment probabilities remain strictly bounded away from zero and one,
these near-violations of overlap create challenging finite-sample behavior,
reflecting difficulties commonly encountered in applied causal inference.

\section{Experimental Details}
\label{app:exp-details}

This appendix provides complete details of the
network architectures, training procedures, and
evaluation protocols used in our experiments.

\subsection{PseudoCode}

 \begin{algorithm2e}[htbp]
    \caption{Precision-Corrected PNS Bounds}
    \label{alg:pns_bounds}
    \small % Retaining your small font size
    
    % Define Input/Output keywords to match your style
    \SetKwInOut{Input}{Require}
    \SetKwInOut{Output}{Ensure}
    
    \Input{$\mathcal{D}_{\text{obs}}$, $\mathcal{D}_{\text{exp}}$, $z$, $M$, $1-\alpha$}
    \Output{$[\hat{L}_\alpha(z), \hat{U}_\alpha(z)]$}
    
    % \nonl suppresses the line number for this specific line (like \Statex)
     \textbf{Train \& Sample:}\;
    Train ENN $f_\theta(z, \zeta)$\;
    \For{$m = 1, \ldots, M$}{
        $\zeta_m \sim P_\zeta$\;
        $(p_{xy}^{(m)}, \mu_x^{(m)}) \gets f_\theta(z, \zeta_m)$\;
        Compute $\boldsymbol{\ell}^{(m)}(z)$, $\mathbf{u}^{(m)}(z)$\;
    }
    
     \textbf{Statistics:}\;
    $\bar{\ell}_j, s_{\ell_j} \gets$ mean, std of $\{\ell_j^{(m)}\}$\;
    $\bar{u}_m, s_{u_m} \gets$ mean, std of $\{u_m^{(m)}\}$\;

     \textbf{Critical Values:}\;
    $\tilde{W}^{L,(m)} \gets \max_j \frac{\ell_j^{(m)} - \bar{\ell}_j}{s_{\ell_j}}$\;
    $\kappa_\alpha^L \gets (1\!-\!\alpha)$-quantile of $\{\tilde{W}^{L,(m)}\}$\;
    $\kappa_\alpha^U \gets$ (analogous)\;

     \textbf{Output:}\;
    $\hat{L}_\alpha \gets \max_j [\bar{\ell}_j - \kappa_\alpha^L s_{\ell_j}]$\;
    $\hat{U}_\alpha \gets \min_m [\bar{u}_m + \kappa_\alpha^U s_{u_m}]$\;
    \KwRet $[\hat{L}_\alpha(z), \hat{U}_\alpha(z)]$\;
    
    \end{algorithm2e}
\subsection{Network Architecture}
\label{app:architecture}

\paragraph{Base model.}
All methods share a common
\texttt{AnchoredPNSModel} consisting of a
backbone MLP and three output heads.  The
backbone is a depth-3 MLP with ReLU activations
and uniform hidden width:
\[
  \phi_\theta(z) \;=\;
    W_4 \, \sigma\!\bigl(
      W_3 \, \sigma\!\bigl(
        W_2 \, \sigma(W_1 z + b_1) + b_2
      \bigr) + b_3
    \bigr) + b_4
    \;\in\; \mathbb{R}^{128},
\]
with layer dimensions
$d_{\mathrm{in}} \to 128 \to 128 \to 128
\to 128$ (three hidden layers of width~128
plus a linear output layer), where $\sigma$
denotes ReLU.

\paragraph{Output heads.}
The joint probability head applies a linear
projection followed by softmax:
\[
  (p_{00}, p_{01}, p_{10}, p_{11})(z) \;=\;
    \mathrm{softmax}\!\bigl(
      W_{\mathrm{joint}} \, \phi_\theta(z)
      + b_{\mathrm{joint}}
    \bigr)
    \;\in\; \Delta^3.
\]
Two scalar auxiliary heads output unconstrained
logits
$\delta_x(z) = w_x^\top \phi_\theta(z) + c_x$
for $x \in \{0,1\}$, which are transformed via
the Tian--Pearl constraint-enforcing
reparameterization:
\begin{align}
  \mu_0(z) &= p_{01}(z)
    + \bigl(1 - p_{00}(z) - p_{01}(z)\bigr)
      \cdot \sigma_{\mathrm{s}}(\delta_0(z)),
    \label{eq:mu0} \\
  \mu_1(z) &= p_{11}(z)
    + \bigl(1 - p_{10}(z) - p_{11}(z)\bigr)
      \cdot \sigma_{\mathrm{s}}(\delta_1(z)),
    \label{eq:mu1}
\end{align}
where $\sigma_{\mathrm{s}}$ denotes the sigmoid
function.  This ensures
$\mu_x(z) = P(Y{=}1 \mid \mathrm{do}(X{=}x),
Z{=}z)$ satisfies the Tian--Pearl bounds by
construction.

\paragraph{ENN hypermodel.}
For the epistemic neural network (ENN), we use a
hypermodel~\citep{osband2023epistemic} that
generates the \emph{full parameter vector}
$\theta$ of the base model from an epistemic
index $\zeta \sim \mathcal{N}(0, I_{16})$:
\[
  \theta(\zeta) = h_\psi(\zeta),
  \qquad
  h_\psi \colon \mathbb{R}^{16}
    \to \mathbb{R}^{|\theta|},
\]
where $h_\psi$ is an MLP with two hidden layers
of width~64 and ReLU activations.
An additive prior network with the same
hypermodel architecture but \emph{frozen}
parameters and hidden width~32 provides a fixed
epistemic prior, scaled by
$\alpha_{\mathrm{prior}} = 1.0$.

\subsection{Training}
\label{app:training}

\paragraph{Loss function.}
The model output is converted to 8 log-logits:
$[\log p_{00}, \log p_{01}, \log p_{10},
  \log p_{11},
  \log(1{-}\mu_0), \log \mu_0,
  \log(1{-}\mu_1), \log \mu_1]$.
The joint probability head is trained on
observational data using cross-entropy:
\[
  \mathcal{L}_{\mathrm{obs}}
  = -\frac{1}{n_{\mathrm{obs}}}
    \sum_{i}
    \sum_{x,y}
    \mathbf{1}[X_i {=} x, Y_i {=} y]
    \,\log p_{xy}(Z_i).
\]
The auxiliary heads are trained on experimental
data using binary cross-entropy:
\[
  \mathcal{L}_{\mathrm{exp}}
  = -\frac{1}{n_{\mathrm{exp}}}
    \sum_{i}
    \bigl[
      Y_i \log \mu_{X_i}(Z_i)
      + (1{-}Y_i) \log(1{-}\mu_{X_i}(Z_i))
    \bigr].
\]
The total loss is
$\mathcal{L} = \mathcal{L}_{\mathrm{obs}}
  + \mathcal{L}_{\mathrm{exp}}$,
with gradient updates from
$\mathcal{L}_{\mathrm{exp}}$ blocked from
propagating through the $p_{xy}$ softmax via
\texttt{stop\_gradient}, so the observational
head is trained only on
$\mathcal{L}_{\mathrm{obs}}$.

\paragraph{ENN training.}
The loss is averaged over epistemic index
samples drawn from the hypermodel's indexer.
We use Adam with learning rate
$3 \times 10^{-3}$ (constant schedule), batch
size $8{,}192$, and train for $800$ epochs.
Input features are standardized (zero mean, unit
variance) using the joint training set
statistics.  The training set is split 80/20
(train/validation), with validation loss
evaluated every 400 epochs.  No early stopping
or weight decay is used.\footnote{%
  The ENN library's \texttt{Experiment.train}
  calls \texttt{optimizer.update(grads,
  opt\_state)} without passing \texttt{params},
  so \texttt{optax.adamw} weight decay is
  silently skipped; regularization must be
  applied through the loss function.}

\paragraph{Multiplier bootstrap training.}
The base \texttt{AnchoredPNSModel} is trained
with Adam (learning rate $10^{-3}$), batch
size~$256$, for $100$ epochs with a 90/10
train/validation split.  Observational and
experimental data are concatenated into a joint
dataset, with per-sample masks determining
which loss terms are active.

\subsection{Hyperparameter Configuration}
\label{app:hyperparams}

Hyperparameters were tuned using only
validation loss and expected calibration error
(ECE) of the predicted probabilities, without
access to ground-truth interventional
probabilities or PNS values.  This reflects the
realistic setting in which the practitioner
observes only the training data and standard
calibration diagnostics.

Tables~\ref{tab:enn-hparams}
and~\ref{tab:mb-hparams} report the final
configurations for each method and dataset.

\begin{table}[h]
\centering
\caption{ENN hypermodel hyperparameters.}
\label{tab:enn-hparams}
\small
\begin{tabular}{lcc}
\toprule
 & \textbf{Synthetic} & \textbf{ACIC} \\
\midrule
\multicolumn{3}{l}{\textit{Architecture}} \\
Model variant
  & \texttt{HyperNet}
  & \texttt{HyperNet} \\
Input dimension $d_{\mathrm{in}}$
  & 15 & 200 \\
Backbone hidden dim
  & 128 & 256 \\
Backbone depth
  & 3 & 3 \\
Epistemic index dim
  & 20 & 18 \\
Hypernetwork hidden
  & $(64, 64)$
  & $(32, 32)$ \\
Prior hypernetwork hidden
  & $(32, 32)$
  & $(32, 32)$ \\
Prior scale $\alpha_{\mathrm{prior}}$
  & 1.0 & 1.25 \\
\midrule
\multicolumn{3}{l}{\textit{Training}} \\
Optimizer
  & Adam & Adam \\
Learning rate
  & $3 \times 10^{-4}$
  & $3 \times 10^{-5}$ \\
LR schedule
  & constant & constant \\
Batch size
  & 8{,}192
  & 8{,}192 \\
Epochs
  & 800 & 4{,}600 \\
Validation split
  & 0.2 & 0.2 \\
L2 regularization
  & 0 & 0 \\
Standardize inputs
  & yes & yes \\
\midrule
\multicolumn{3}{l}{\textit{Inference}} \\
Posterior samples $M$
  & 8{,}000
  & 5{,}000 \\
CLR quantile
  & 0.975 & 0.975 \\
\midrule
\multicolumn{3}{l}{\textit{Data}} \\
$n_{\mathrm{obs}}$
  & 100{,}000
  & 2{,}400{,}000 \\
$n_{\mathrm{exp}}$
  & 50{,}000
  & 1{,}000{,}000 \\
$n_{\mathrm{test}}$
  & 3{,}000 & 2{,}000 \\
Monte Carlo replicates
  & 1000 & 1000 \\
\bottomrule
\end{tabular}
\end{table}

\begin{table}[h]
\centering
\caption{Multiplier bootstrap hyperparameters.}
\label{tab:mb-hparams}
\small
\begin{tabular}{lcc}
\toprule
 & \textbf{Synthetic} & \textbf{ACIC} \\
\midrule
\multicolumn{3}{l}{\textit{Architecture}} \\
Input dimension $d_{\mathrm{in}}$
  & 15 & 200 \\
Backbone hidden dim
  & 128 & 128 \\
Backbone depth
  & 3 & 3 \\
\midrule
\multicolumn{3}{l}{\textit{Training}} \\
Optimizer
  & Adam & Adam \\
Learning rate
  & $10^{-3}$
  & $10^{-3}$ \\
Batch size
  & 256 & 256 \\
Epochs
  & 100 & 100 \\
Validation split
  & 0.1 & 0.1 \\
\midrule
\multicolumn{3}{l}{\textit{Inference}} \\
Bootstrap replicates $B$
  & 1{,}000
  & 1{,}000 \\
Significance level $\alpha$
  & 0.05 & 0.05 \\
Influence mode
  & last-layer
  & last-layer \\
CG iterations
  & 50 & 50 \\
CG damping
  & $10^{-4}$
  & $10^{-4}$ \\
\midrule
\multicolumn{3}{l}{\textit{Data}} \\
$n_{\mathrm{obs}}$
  & 100{,}000
  & 100{,}000 \\
$n_{\mathrm{exp}}$
  & 20{,}000
  & 20{,}000 \\
$n_{\mathrm{test}}$
  & 1{,}000
  & 1{,}000 \\
\bottomrule
\end{tabular}
\end{table}

\subsection{Inference}
\label{app:inference}

\paragraph{ENN posterior inference.}
At test time, for each covariate $z$, we draw
$M = 20$ epistemic indices
$\{\zeta_m\}_{m=1}^{M}$ and compute the model
output for each, obtaining samples of the
observational and interventional probability
estimates.  Standard errors are computed across
posterior samples for each of the four
lower-bound and four upper-bound terms:
\[
  s_{\ell}(j)
  = \sqrt{
    \frac{1}{M-1}
    \sum_{m=1}^{M}
    \bigl(\ell^{(m)}(j) - \bar{\ell}(j)\bigr)^2
  }.
\]

\paragraph{CLR intersection bounds.}
Following
\citet{chernozhukov2013intersection}, we
compute standardized test statistics
\[
  W_z^L = \max_j
    \frac{\ell^{(m)}(j) - \bar{\ell}(j)}
         {s_\ell(j) + \epsilon},
  \qquad
  W_z^U = -\min_j
    \frac{u^{(m)}(j) - \bar{u}(j)}
         {s_u(j) + \epsilon},
\]
with $\epsilon = 10^{-10}$.  Critical values
$\kappa^L$ and $\kappa^U$ are the $97.5\%$
empirical quantiles of $W_z^L$ and $W_z^U$
across posterior samples, respectively.  The
precision-corrected bounds are:
\begin{align}
  \hat{L}(z)
    &= \max_j
      \bigl[\bar{\ell}(j) - \kappa^L
        \cdot s_\ell(j)\bigr],
    \\
  \hat{U}(z)
    &= \min_j
      \bigl[\bar{u}(j) + \kappa^U
        \cdot s_u(j)\bigr].
\end{align}

\paragraph{Multiplier bootstrap inference.}
Influence functions are computed using the
trained model's per-sample gradients and a
conjugate-gradient approximation to the inverse
Hessian (50~CG iterations, damping $10^{-4}$).
At inference, \texttt{stop\_gradient} is
\emph{disabled} so that the full Jacobian of
the bound terms w.r.t.\ all model parameters
is captured.  For each test point,
$B = 1{,}000$ multiplier bootstrap replicates
are drawn:
\[
  T_b^{*,L} = \max_j
    \frac{\xi_b^\top \psi_j^L}
         {\sqrt{n}\,s_j^L},
  \qquad
  \xi_b \sim \mathcal{N}(0, I_n),
\]
where $\psi_j^L$ are per-sample influence
functions for lower-bound term~$j$.  Critical
values $\kappa^L$ and $\kappa^U$ are the
$(1 - \alpha/2)$ percentiles with
$\alpha = 0.05$, applying Bonferroni correction
for joint coverage of the two-sided interval.

\subsection{Evaluation Metrics}
\label{app:metrics}

\paragraph{Coverage.}
Point coverage is
$P(\hat{L}(z) \le \mathrm{PNS}(z)
   \le \hat{U}(z))$,
averaged over test points.  Identified-set
coverage is
$P(\hat{L}(z) \le L^*(z) \;\text{and}\;
   \hat{U}(z) \ge U^*(z))$,
where $[L^*(z), U^*(z)]$ is the true
Tian--Pearl identified set.

\paragraph{Interval score.}
The Winkler interval score is
\[
  \mathrm{IS}_i
  = (u_i - \ell_i)
    + \tfrac{2}{\alpha}(\ell_i - y_i)_+
    + \tfrac{2}{\alpha}(y_i - u_i)_+
\]
with $\alpha = 0.05$.  Lower scores indicate
narrower intervals with better coverage.

\paragraph{Baselines.}
S-learner and T-learner baselines use the same
backbone architecture; S-learner includes
treatment as an additional input feature, while
T-learner trains separate networks on treated
and control subgroups.

\subsection{Computational Resources}
\label{app:compute-res}

Experiments were conducted on NVIDIA A100 GPUs
(40\,GB) using float32 precision.
The implementation uses
JAX~\citep{jax2018github} and
Equinox~\citep{kidger2021equinox}.
% Code is available at \texttt{[URL]}.

\section{Evaluation Metrics}\label{app:metrics-2}

We evaluate the estimated PNS bounds using metrics that capture both calibration
and informativeness.

For each test instance $i$, a method outputs a PNS interval
$[\ell_i, u_i]$, and the true PNS value $y_i$ is known from the underlying SCM.

\paragraph{Constraint violation rate.}
The fraction of test instances for which the estimated observational and
interventional probabilities violate at least one compatibility
constraint.

\paragraph{Coverage.}
The fraction of test instances for which the true PNS value lies within the
estimated interval:
\begin{equation}
    \frac{1}{N}\sum_{i=1}^{N}
    \mathbf{1}\!\left[y_i \in [\ell_i, u_i]\right].
\end{equation}

\paragraph{Mean absolute bound width.}
\begin{equation}
    \frac{1}{N}\sum_{i=1}^{N}(u_i - \ell_i),
\end{equation}
which measures the average sharpness of the bounds.

\paragraph{Interval score.}
A proper scoring rule that jointly penalizes miscoverage and excessive width:
\begin{equation}
    \mathrm{IS}_i = (u_i - \ell_i) + \frac{2}{\alpha}(\ell_i - y_i)_+ + \frac{2}{\alpha}(y_i - u_i)_+
\end{equation}
with $\alpha = 0.05$. Lower values indicate better performance.
Metrics are reported on held-out test data and averaged across dataset
replicates.

\section{Computational Comparison}
\label{app:compute}

The multiplier bootstrap derives standard errors from influence functions, which
require forming and solving a Hessian system of size $O(p^2)$ per test point, where
$p$ is the number of network parameters. This yields per-point inference cost that
scales linearly in $N_\mathrm{test}$ and, for the full-network variant, memory that
grows as $n_\mathrm{train} \times p \times 4$ bytes — a hard wall that makes it
inapplicable at the sample sizes required for high-dimensional settings. EpiNet
replaces this with a single batched forward pass through the epinet, with cost
$O(1)$ in both $n_\mathrm{train}$ and $N_\mathrm{test}$.

\textbf{Experimental protocol.} We benchmark all three uncertainty estimators —
full-network multiplier bootstrap (MB Full), last-layer multiplier bootstrap (MB
Last-Layer), and Causal EpiNet — on the Li et al.\ synthetic DGP at
$n_\mathrm{train} = 25\mathrm{k}$ on a Quadro RTX 8000 (46\,GB VRAM, 2$\times$AMD
EPYC 7502, 377\,GB RAM). We report one-time setup cost, per-point inference time at
varying $N_\mathrm{test}$, and scaling behaviour as $n_\mathrm{train}$ increases.

\textbf{Results.} Table~\ref{tab:compute} reports the full benchmark. At
$n_\mathrm{train} = 25\mathrm{k}$, EpiNet processes all test points in a single
batched call of 612\,ms regardless of $N_\mathrm{test}$, while MB Last-Layer requires
811\,ms \emph{per point}: at 1k test points this reaches $\sim$811\,s, and at 10k
points $\sim$2.3 hours. MB Full costs 6,954\,ms per point ($\sim$19.3 hours at 10k
points). As $n_\mathrm{train}$ grows, MB Last-Layer's per-point cost rises to
1,131\,ms at 200k; MB Full runs out of memory above $n_\mathrm{train} = 50\mathrm{k}$
due to the influence cache requiring simultaneous storage of two copies of the
$n_\mathrm{train} \times p$ matrix. EpiNet's inference cost remains constant at
$\sim$612\,ms across all training sizes. The one-time training cost is higher for
EpiNet (38.4\,s vs.\ 18.0\,s for MB Last-Layer), but this is amortised across all
subsequent test evaluations and is negligible at scale. EpiNet is 1.3$\times$ faster
than MB Last-Layer and 11.4$\times$ faster than MB Full at 25k scale, and is the only
method that remains tractable at the sample sizes used in the ACIC experiments.

\begin{table}[h]
\centering
\caption{Computational benchmark on a Quadro RTX 8000 (46\,GB VRAM) at
$n_\mathrm{train} = 25\mathrm{k}$ unless noted.
$^\dagger$EpiNet processes all test points in a single batched forward pass;
cost is $O(1)$ in $N_\mathrm{test}$.}
\label{tab:compute}
\small
\begin{tabular}{lrrr}
\toprule
& MB Last-Layer & MB Full-Network & Causal EpiNet \\
\midrule
\multicolumn{4}{l}{\textit{One-time setup}} \\
\quad Training (s)          & 13.0   & 12.7          & 36.0  \\
\quad Influence cache (s)   & 1.0    & 6.1           & —     \\
\quad Total (s)             & 18.0   & 22.5          & 38.4  \\
\midrule
\multicolumn{4}{l}{\textit{Inference scaling in $N_\mathrm{test}$}} \\
\quad 1 point               & 811\,ms       & 6{,}954\,ms        & \textbf{612\,ms}$^\dagger$ \\
\quad 1k points             & $\sim$811\,s  & $\sim$6{,}954\,s   & \textbf{612\,ms}$^\dagger$ \\
\quad 10k points            & $\sim$2.3\,hr & $\sim$19.3\,hr     & \textbf{612\,ms}$^\dagger$ \\
\midrule
\multicolumn{4}{l}{\textit{Scaling in $n_\mathrm{train}$}} \\
\quad 50k                   & 811\,ms/pt    & 6{,}954\,ms/pt     & \textbf{612\,ms} \\
\quad 200k                  & 1{,}131\,ms/pt & \textsc{oom}      & \textbf{627\,ms} \\
\quad 1M+                   & min/pt        & \textsc{oom}       & \textbf{$\sim$612\,ms} \\
\bottomrule
\end{tabular}
\end{table}

\section{Sensitivity Analysis to experimental data} \label{app:sensitivity}

The Tian–Pearl framework combines two distinct data sources: $\mu_x(z)$ is identified
from experimental data and $p_{xy}(z)$ from observational data, and both are required
for the sharp PNS bounds. In practice, randomised experimental data is typically
scarcer and more expensive to collect. A practically useful inference procedure must
therefore remain calibrated when experimental data is limited and tighten its
intervals as more becomes available — without sacrificing coverage in either regime.

\textbf{Experimental protocol.} We evaluate Causal EpiNet on the 
\citet{li2022probabilitiescausationadequatesize} synthetic
DGP across two sets of regimes. In the first, we fix $n_\mathrm{obs} = 100\mathrm{k}$
and vary $n_\mathrm{exp} \in \{20\mathrm{k}, 50\mathrm{k}, 80\mathrm{k},
100\mathrm{k}\}$, isolating the effect of experimental sample size while holding
observational precision fixed. In the second, we scale both sources jointly at
$(n_\mathrm{obs}, n_\mathrm{exp}) \in \{(100\mathrm{k}, 100\mathrm{k}),\,
(500\mathrm{k}, 500\mathrm{k}),\, (1\mathrm{M}, 1\mathrm{M})\}$ to assess large-scale
behaviour. All configurations use 1000 Monte Carlo replicates. These sample sizes
exceed the memory limits of the full-network multiplier bootstrap, which runs out of
memory above $n_\mathrm{train} = 50\mathrm{k}$; the sensitivity analysis therefore
operates entirely in a regime where EpiNet is the only tractable inference procedure.

\textbf{Results.} Table~\ref{tab:sensitivity} and Figure report
point coverage, ID set coverage, bound width, and interval score across all regimes.
ID set coverage remains within its Monte Carlo confidence interval of the 0.95 nominal
level across every configuration. As $n_\mathrm{exp}$ increases from 20k to 100k with
$n_\mathrm{obs}$ fixed, bound width decreases monotonically from 0.330 to 0.307 and
interval score from 0.350 to 0.326. This reflects the ENN correctly tracking the
information content of the experimental sample: as $\mu_x(z)$ is estimated more
precisely, epistemic uncertainty over the bound components decreases and the
precision-corrected intervals tighten without losing calibration. A miscalibrated
uncertainty estimator would either widen intervals unnecessarily to maintain coverage
or lose coverage as intervals tighten; neither occurs here. In the jointly scaled
regimes, both metrics stabilise beyond 500k/500k with no detectable improvement at
1M/1M. We make no claim about convergence to the oracle identification region, as the
oracle bound width is not available for this DGP; the stabilisation is an empirical
observation about this estimator at these scales.

\begin{table}[htbp]
\centering
\caption{Sensitivity to experimental and observational sample size on the Li et al.\ synthetic DGP. All results averaged over 700 Monte Carlo replicates. 95\% Monte Carlo confidence intervals in brackets.}
\label{tab:sensitivity}
\small
\begin{tabular}{llcccc}
\toprule
$n_\mathrm{obs}$ & $n_\mathrm{exp}$ & Point Coverage & ID Set Coverage & Bound Width & Interval Score \\
\midrule
\multicolumn{6}{c}{\textit{Fixed $n_\mathrm{obs} = 100\mathrm{k}$, varying $n_\mathrm{exp}$}} \\
100k & 20k  & 0.983 {\scriptsize [0.979, 0.985]} & 0.953 {\scriptsize [0.949, 0.956]} & 0.330 {\scriptsize [0.325, 0.332]} & 0.350 {\scriptsize [0.347, 0.354]} \\
100k & 50k  & 0.982 {\scriptsize [0.979, 0.985]} & 0.956 {\scriptsize [0.953, 0.958]} & 0.318 {\scriptsize [0.316, 0.321]} & 0.336 {\scriptsize [0.333, 0.338]} \\
100k & 80k  & 0.979 {\scriptsize [0.976, 0.983]} & 0.955 {\scriptsize [0.951, 0.959]} & 0.310 {\scriptsize [0.307, 0.313]} & 0.330 {\scriptsize [0.327, 0.334]} \\
100k & 100k & 0.977 {\scriptsize [0.974, 0.981]} & 0.951 {\scriptsize [0.948, 0.955]} & 0.307 {\scriptsize [0.305, 0.310]} & 0.326 {\scriptsize [0.323, 0.329]} \\
\midrule
\multicolumn{6}{c}{\textit{Balanced scaling}} \\
100k  & 100k  & 0.977 {\scriptsize [0.974, 0.981]} & 0.951 {\scriptsize [0.948, 0.955]} & 0.307 {\scriptsize [0.305, 0.310]} & 0.326 {\scriptsize [0.323, 0.329]} \\
500k  & 500k  & 0.977 {\scriptsize [0.974, 0.981]} & 0.954 {\scriptsize [0.950, 0.958]} & 0.307 {\scriptsize [0.305, 0.309]} & 0.326 {\scriptsize [0.322, 0.330]} \\
1M    & 1M    & 0.978 {\scriptsize [0.975, 0.982]} & 0.955 {\scriptsize [0.951, 0.958]} & 0.306 {\scriptsize [0.305, 0.309]} & 0.325 {\scriptsize [0.323, 0.329]} \\
\bottomrule
\end{tabular}
\end{table}

\end{document}